\documentclass{article}
\pdfoutput=1

\usepackage[utf8]{inputenc} 
\usepackage[T1]{fontenc}    
\usepackage{pifont}
\pdfoutput=1

\usepackage[left=1in,right=1in,]{geometry}
\usepackage{multicol,lipsum,xcolor}
\usepackage[hidelinks]{hyperref}       
\usepackage{url}            
\usepackage{amsmath,amsthm,amssymb,bm}
\usepackage{bbold}
\DeclareMathOperator{\argmin}{argmin}  
\DeclareMathOperator{\argmax}{argmax}
\usepackage{algorithm,algorithmic}
\usepackage{caption}
\usepackage{graphicx}
\usepackage{enumerate}
\usepackage{enumitem}
\usepackage[sort, square]{natbib}
\usepackage{parskip}
\usepackage{comment}
\usepackage{microtype}
\usepackage{diagbox}
\usepackage{tikz}
\usepackage{bbm}
\usepackage{thm-restate}
\usepackage{enumitem}

\usepackage{fancyhdr}

\newtheorem{thm}{Theorem}
\newtheorem*{thm*}{Theorem}
\newtheorem{lemma}[thm]{Lemma}

\theoremstyle{definition}

\newtheorem{defini}{Definition}

\newtheorem{assumption}{Assumption}
\newtheorem*{remark}{Remark}

\newenvironment{itemize*}%
{\begin{itemize}[leftmargin=*,topsep=0pt]%
		\setlength{\itemsep}{0pt}%
		\setlength{\parskip}{0pt}}%
	{\end{itemize}}
\newenvironment{enumerate*}%
{\begin{enumerate}[leftmargin=*,topsep=0pt]%
		\setlength{\itemsep}{0pt}%
		\setlength{\parskip}{0pt}}%
	{\end{enumerate}}
	
\newcommand\numberthis{\addtocounter{equation}{1}\tag{\theequation}}

\newcount\Comments  
\definecolor{darkred}{rgb}{0.7,0,0}
\definecolor{teal}{rgb}{0.3,0.8,0.8}
\definecolor{forestgreen}{rgb}{0.13, 0.55, 0.13}
\newcommand{\kibitz}[2]{\ifnum\Comments=1{\textcolor{#1}{\textsf{\footnotesize #2}}}\fi}

\newcommand{\prob}[0]{\mathbb{P}}
\newcommand{\E}[0]{\mathbb{E}}

\newcommand{\Sset}[0]{\mathcal{S}}
\newcommand{\U}[0]{\mathcal{U}}
\newcommand{\A}[0]{\mathcal{A}}
\newcommand{\B}[0]{\mathcal{B}}
\newcommand{\D}[0]{\mathcal{D}}
\newcommand{\M}[0]{\mathcal{M}}
\newcommand{\err}[0]{\text{err}}
\newcommand{\loss}[0]{\text{loss}}

\newcommand{\ind}[0]{\mathbb{1}}

\title{Offline Policy Evaluation and Optimization
under Confounding}

\author{Chinmaya Kausik$^1$*, Yangyi Lu$^3$*,  Kevin Tan$^2$*\\
Maggie Makar$^1$, Yixin Wang$^1$, Ambuj Tewari$^1$}
\date{\small $^1$University of Michigan\ \ $^2$University of Pennsylvania\ \ $^3$Pinterest}

\begin{document}

\maketitle

\begin{abstract}

Evaluating and optimizing policies in the presence of unobserved confounders is a problem of growing interest in offline reinforcement learning. Using conventional methods for offline RL in the presence of confounding can not only lead to poor decisions and poor policies, but also have disastrous effects in critical applications such as healthcare and education. We map out the landscape of offline policy evaluation for confounded MDPs, distinguishing assumptions on confounding based on whether they are memoryless and on their
effect on the data-collection policies. We characterize settings where consistent value estimates are provably not achievable, and provide algorithms with guarantees to instead estimate lower bounds on the value. When consistent estimates are achievable, we provide algorithms for value estimation with sample complexity guarantees. We also present new algorithms for offline policy improvement and prove local convergence guarantees. Finally, we experimentally evaluate our algorithms on both a gridworld environment and a simulated healthcare setting of managing sepsis patients. In gridworld, our model-based method provides tighter lower bounds than existing methods, while in the sepsis simulator, our methods significantly outperform confounder-oblivious benchmarks.
\end{abstract}

\section{Introduction}
A central problem in sequential decision making is learning from offline data, since collecting data in an online fashion is often prohibitively expensive or unsafe \citep{levine2020survey}. Since real-life data is often affected by latent variables, there has been a rise of interest in formulations of reinforcement learning problems with hidden information \citep{Nair2021opePOMDP, Miao2022opePOMDP, lingxiao2020dovi}. The most general kind of latent information is considered by partially observable MDPs or POMDPs \citep{kaelbling1998pomdps, tennenholtz2019pomdp}, where the latent information can affect both rewards and transitions. However, the reward is often \textit{designed} by the user based only on observable variables. In medical examples, the reward could be given based on observed vitals, but unrecorded genetic conditions and socio-economic status can affect actions taken and future states. These examples motivate the important case of reinforcement learning with unobserved confounders, defined as latent information that affects transitions, but not rewards\footnote{Some papers define confounders using a kind of "memorylessness," and allow them to affect rewards \citep{Zhang2016MDPUC,lingxiao2020dovi}. We only consider unconfounded rewards.} \citep{kallus2020confounding,bruns2021model, brunssmith2023robust}. 

The hardness of learning from offline data under confounding comes from the fact that partially observed transitions can be further obscured by behavior policies that might have known the unrecorded confounder \citep{kallus2020confounding}.\ Two offline data distributions might thus be identical despite coming from different confounded MDPs, if the behavior policies accommodated for this difference (see Theorem~\ref{thm:iid-lower-bound}).

\begin{table*}[t]
\centering
\begin{tabular}{  m{8.3em} | m{5cm}| m{4.8cm}  } 
 & With Sensitivity Constraint & Without Sensitivity Constraint \\ 
  \hline
 Memoryless Confounders & 
 Consistency not possible (Theorem~\ref{thm:iid-lower-bound}, $\Omega(\varepsilon H)$ error lower bound), $O(\varepsilon H^2)$ error upper bound with 3 methods (Theorems~\ref{thm:err_FQE}, \ref{thm:err_CFQE}, \ref{thm:model_base}) & 
 $\Omega(H)$ error lower bound (Theorem~\ref{thm:iid-lower-bound})\\ 
 \hline
 Confounders with Memory & 
 Methods mentioned above have $\Omega(H)$ error lower bounds, even with unconfounded $\pi_b$ and $\pi_e$ (Theorem~\ref{thm:hist-dependent-lower-bound}) & 
$\Omega(H)$ lower bound in general. \newline For global confounders, consistency possible, sample complexity guarantees given (Theorem~\ref{thm:cluster-ope-sample-complexity})  \\ 

 \hline
\end{tabular}
\caption{Hardness of the OPE problem under different assumptions on the nature of confounding present. $\Gamma$ is a so-called sensitivity parameter, with $\Gamma = 1+ O(\varepsilon)$. Higher $\varepsilon$ corresponds to more confounded $\pi_b$.}
\label{table:hardness}
\end{table*}

To provide guarantees for learning from offline data, the most common assumption in previous work is that confounders are "memoryless" (Assumption~\ref{assump:hist_indep_u}). This assumption essentially means that they are sampled afresh at each step independently of past confounders, states, or actions \citep{brunssmith2023robust}. In many real-life applications like healthcare and epidemiology \citep{Daniel2013timevarying, clare2018timevarying, timevarying1, timevarying2}, it is more appropriate to assume that the confounders are sampled "with memory" of previous confounders, and even states and actions. A lot of work also assumes that behavior policies follow a sensitivity constraint (Assumption~\ref{assump:sensitivity}) \citep{kallus2020confounding, bruns2021model}. Motivated by these observations, we take the first step towards providing a structured view of the landscape of offline RL for confounded MDPs, distinguishing settings in terms of sensitivity assumptions and whether confounders have memory. We also introduce and study an important sub-case of confounders with memory, called global confounders (Assumption~\ref{assump:global_u}). Specifically, we ask the following questions for each setting:
\vspace{0cm}
\begin{enumerate}[nosep]
    \item[Q.1.] \textit{If consistent offline policy estimation (OPE) is not possible, can we prove lower bounds on the error? What guarantees can we give for algorithms that instead estimate bounds on the value?}
    \item[Q.2.] \textit{If consistent OPE is possible, then what algorithms achieve this? What is their sample complexity?} 
    \item[Q.3.] \textit{How can we use these insights for offline policy improvement?}
\end{enumerate}
\vspace{-0.1cm}
\paragraph{Paper Structure and Contributions.} We detail our contributions below. A summary of key results is provided in Table~\ref{table:hardness}. 

\textit{OPE for Memoryless Confounders, Section~\ref{sec:history-ind-confounders}:} In Theorem~\ref{thm:iid-lower-bound}, we give the first lower bound for OPE error that depends on a sensitivity parameter $\Gamma$ and horizon length $H$. By choosing $\Gamma$ appropriately, we show that value estimation can be \emph{arbitrarily} bad without a sensitivity constraint. The theorem also \emph{quantitatively} shows that the lower bound on error grows with $H$ and consistent estimates are not possible, even under a sensitivity constraint. To provide algorithms that estimate lower bounds on the value, we modify the CFQE algorithm due to \citep{bruns2021model} to our more general definition of memoryless confounding. We are the first to compute quantitative upper bounds on its error and the error for FQE, in Theorems~\ref{thm:err_FQE} and \ref{thm:err_CFQE}. We further provide a new model-based algorithm that improves over CFQE for stationary transition structures, and provide guarantees for it in Theorems~\ref{thm:model_base} and~\ref{thm:model-based-consistent}.

\textit{OPE for Confounders with Memory, Section~\ref{sec:history-dependent-confounders}:} While FQE is a standard workhorse for OPE and also enjoys guarantees for memoryless confounders, it is unclear if (and how badly) FQE fails for confounders with memory. In particular, it is non-trivial to produce lower bound examples in this case. We are the first to present one in Theorem~\ref{thm:hist-dependent-lower-bound}, where we show that FQE can have \emph{arbitrarily} large error for confounders with memory, even for unconfounded $\pi_b$ and $\pi_e$ with bounded concentrability. This shows the hardness of OPE for \emph{general} confounders with memory. In this light, we introduce and study the important sub-case of \textit{global confounders}, where the confounder is fixed at the beginning of each trajectory. We leverage the work of \citep{ambuj2022mixmdp} on clustering mixtures of MDPs to provide an algorithm for OPE under this assumption, along with sample complexity guarantees in Theorem~\ref{thm:cluster-ope-sample-complexity}. While past work on confounded RL has focused only on consistency, we are the first to address the sample complexity of OPE under confounding. 

\textit{Offline Policy Improvement, Section~\ref{sec:optimization}:} We address offline policy improvement in Section \ref{sec:optimization}, presenting policy gradient methods for memoryless confounders under a sensitivity assumption, as well as for global confounders. We prove local convergence for both.

\textit{Experiments, Section~\ref{sec:experiments}:} We test and compare OPE methods for memoryless confounders in the gridworld environment provided by \citep{bruns2021model}. Our experiments show that our model-based method gives tighter lower bounds than existing methods. We also successfully run our policy gradient method for memoryless confounders in the same environment. OPE and policy gradient methods for global confounders are tested in the sepsis simulator from \citep{oberst2019counterfactual}, where we significantly outperform confounder-oblivious implementations of both FQE and policy gradients.

\paragraph{Related Work.} Many specific assumptions on confounders have been studied in recent literature. \citep{kallus2020confounding, bruns2021model, namkoong2020off} all provide algorithms that estimate bounds on the value under a sensitivity assumption. The first two assume variants of memorylessness, while the third assumes that the confounding occurs during only a single timestep. Other work like \citep{kallus2020prox} uses a latent variable model for states and actions to get consistent point estimates. This is similar to work in the POMDP setting \citep{tennenholtz2019pomdp}, and neither approach directly applies to our settings. In general, a treatment of confounders with memory and a big-picture view of the OPE problem under confounding is still missing.

On the other hand, literature on offline policy {\em improvement} in the presence of confounders has grown more gradually. \citep{brunssmith2023robust} provide robust fitted-Q-iteration methods under a sensitivity model and a memoryless assumption. This does not apply to confounders with memory, like global confounders. Other work like \citep{lingxiao2020dovi, liao2021instrumental, fu2022offline} uses auxiliary variables from the data to adjust for confounding bias. However, these do not directly apply to our settings.

\section{Setup and Assumptions}\label{sec:possibility-confounding}
\subsection{Background}

We define an episodic confounded MDP by a tuple $(\mathcal{S}\times\mathcal{U}, \mathcal{A}, H, \{\prob_h\}_{h=1}^H, r, d_0)$, described as follows. 
$\mathcal{S}$ is the set of $S$ observed states and $\mathcal{U}$ the set of $U$ unobserved confounders; 
$\mathcal{A}$ is the set of $A$ actions; 
$H$ is the horizon of each episode; 
$d_0$ is the distribution for initial states $(s_1,u_1)\sim d_0$; 
$r: \mathcal{S}\times\mathcal{A}\rightarrow[0,1]$ denotes the reward function; and $\prob_h(s',u'\mid s,u,a)$ denotes the state transition probability at timestep $h$.

The data is collected under a behavior policy $\pi_b$ specified by $\pi_{b,h}(a \mid s, u)$, which might have used the unrecorded confounders and been time-dependent. The observed behavior policy is obtained by marginalizing $u$ over the induced distribution at timestep $h$, and is called $\pi_{b,h}(a \mid s)$. The goal is to estimate the value function $V^{\pi_e}_1$ of a possibly time-dependent evaluation policy $\pi_e$ \emph{that does not use confounders} \cite{bruns2021model}. This is motivated by the fact that confounders can be harder to observe and account for during deployment.

\subsection{Assumptions on Sensitivity and Memory}

We consider two kinds of assumptions on unobserved confounders. The first is whether they "have memory." We define memoryless confounders below to be sampled afresh at each step \cite{brunssmith2023robust}. A memoryless confounder in a healthcare application could be an accident encountered mid-treatment, or in an economics application could be a supply shock affecting the price of oil, as \cite{bruns2021model} highlights. 

\begin{assumption}[Memoryless Confounders]
	\label{assump:hist_indep_u}
	At each timestep $h$, we draw a fresh confounder $u_h \sim P_h(u \mid s=s_h)$, possibly dependent on the current state $s_h$, but independent of past confounders, states and actions.
\end{assumption}

On the other hand, confounders with memory could depend on all past $(s,a,u)$ tuples. We introduce an important sub-case of this, which we call the {\em global confounder} assumption. This is an extreme case of confounders with memory, where the confounder is not just dependent on, but the \emph{same} as all past confounders in the trajectory. In the example of healthcare applications, this could be an unrecorded patient demographic characteristic or genetic condition that does not change over the course of treatment. 

\begin{assumption}[Global Confounders]
	\label{assump:global_u}
	A global confounder is generated by $u \sim P(u)$ at the beginning of an episode, and remains unchanged throughout the episode.
\end{assumption}

A commonly-used assumption for the effect of confounder on $\pi_b$ is a sensitivity model found in \cite{bruns2021model, kallus2020confounding, namkoong2020off}. Note that $\Gamma = 1$ below corresponds to the case where $\pi_b$ is \emph{confounder-oblivious}, that is, independent of the confounder.

\begin{assumption}[Confounding Sensitivity Model]
	\label{assump:sensitivity}
	Given $\Gamma\geq 1$, for all $s\in\Sset, u\in\U$, $h \in \{1, 2, \cdots H\}$ and $a\in\A$:
	\begin{align*}
	\frac{1}{\Gamma} \leq \left(\frac{\pi_{b,h}(a\mid s,u)}{1-\pi_{b,h}(a\mid s,u)}\right) / \left(\frac{\pi_{b,h}(a\mid s)}{1-\pi_{b,h}(a\mid s)}\right) \leq \Gamma,
	\end{align*}
	where $\pi_{b,h}(a\mid s) = \sum_u P_h(u \mid s)\pi_{b,h}(a\mid s,u)$ is the marginalized (observed) behavior policy.
	The above inequality implies the bounds
	$\alpha_h(s,a) \leq \frac{\pi_{b,h}(a\mid s)}{\pi_{b,h}(a\mid s,u)} \leq \beta_h(s,a)$,
	where 
	$\alpha_h(s,a) := \pi_{b,h}(a\mid s) + \frac{1}{\Gamma}(1-\pi_{b,h}(a\mid s))$ and $\beta_h(s,a) := \Gamma + \pi_{b,h}(a\mid s)(1-\Gamma)$.
\end{assumption}

\section{OPE under Memoryless Confounders}\label{sec:history-ind-confounders} 
We discuss OPE when confounders are memoryless. We first open with a result showing that in the absence of a sensitivity assumption like Assumption~\ref{assump:sensitivity}, we can incur an estimation error as bad as $\Omega(H)$. Note that the value functions lie in the range $[0,H]$, so the worst possible OPE error is $H$.

\begin{restatable}[Lower Bound for Memoryless Confounders]{thm}{iidLowerBound}\label{thm:iid-lower-bound}
    There exists a parameter $\varepsilon$ that determines a pair of confounded MDPs $\M_1$ and $\M_2$ with i.i.d. (and thus memoryless) confounders along with stationary policies $\pi_{b_1}$, $\pi_{b_2}$ and $\pi_e$, so that data collected from $\M_i$ using $\pi_{b_i}$ has the same distribution for $i=1,2$, but the values under $\pi_e$ differ by $|V^{\pi_e}_1(\M_1) - V^{\pi_e}_1(\M_2)| = 2\varepsilon H$. In particular, when $\varepsilon = \frac{1}{2} - \frac{1}{H^2}$, the values under $\pi_e$ differ by $\Omega(H)$.
\end{restatable}

It can be seen from the proof of the theorem in Appendix~\ref{sec:iid-lower-bounds} that when $\varepsilon = \frac{1}{2} - \frac{1}{H^2}$, $\Gamma = \Omega(H^2)$. It is then clear that a bound on the sensitivity is necessary. The proof shows that for small $\varepsilon$ in our example, $\Gamma = 1 + O(\varepsilon)$. In this light, even with a sensitivity constraint of $1 + O(\varepsilon)$, we cannot get a consistent estimate of the value of a policy. This is because by Theorem~\ref{thm:iid-lower-bound}, even two observationally indistinguishable confounded MDPs can differ in value under a new $\pi_e$ by $\Omega(\varepsilon H)$. 

Thus, even with infinite data, we can only hope for \textit{bounds} on the value, and the minimum-possible error deteriorates with horizon $H$. We now analyze and present algorithms for obtaining such bounds. 

\subsection{FQE and Confounded FQE}

Fitted Q-Evaluation (FQE), which we recall in Appendix~\ref{sec:FQE_cFQE}, is a standard workhorse for OPE. We first present a new result on the estimation error of FQE under memoryless confounding, proved in Appendix~\ref{sec:proofs_FQE_cFQE}.

\begin{restatable}[FQE Error]{thm}{errFQE}
    \label{thm:err_FQE}
	Suppose $\Gamma = 1 + \varepsilon$ in Assumption~\ref{assump:sensitivity}. Then in the limit of infinite samples, the point estimate $\hat{f}_1(s,a)$ of the Q-function produced by FQE has a worst-case error of
	$|V_1^{\pi_e}(s) - \sum_a\pi_{e,1}(a\mid s)\hat{f}_1(s,a)| = O({\varepsilon} H^2)$ for small $\varepsilon$.

\end{restatable}

Note that FQE gives a point estimate instead of a lower bound on the value function. For many safety-critical applications, it is important to have conservative lower bounds for policy estimation. Using the proof of Theorem~\ref{thm:err_FQE}, we can produce a straightforward lower bound of $\sum_a\pi_{e,1}(a\mid s)\hat{f}_1(s,a) - k\varepsilon H^2$ on the value function, for some $k$ depending on $\varepsilon$. However, this is a worst-case, data-oblivious lower bound. We note that we can get a sharper lower bound using confounded FQE (CFQE), introduced by \cite{bruns2021model} for i.i.d. confounders. Confounded FQE gives a lower bound on the value by sequentially searching for the \emph{worst possible policies} that are consistent with the data and the sensitivity assumption. We adapt it to general memoryless confounders and describe it in Appendix~\ref{sec:FQE_cFQE}. We also provide a new theoretical guarantee for the worst-case error of CFQE below, proved in Appendix~\ref{sec:proofs_FQE_cFQE}.

\begin{restatable}[CFQE Error]{thm}{errCFQE}
\label{thm:err_CFQE}
Suppose $\Gamma = 1+\varepsilon$ in Assumption~\ref{assump:sensitivity}. Then the worst-case error for the lower bound $\hat{f}_1(s,a)$ generated by CFQE in the infinite-sample case is $|V_1^{\pi_e}(s) - \sum_a\pi_{e,1}(a\mid s)\hat{f}_1(s,a)| = O(\varepsilon H^2)$ for any range of $\varepsilon$.
\end{restatable}

Although it has the same \emph{worst-case} error as FQE, we note that CFQE provides an \emph{instance-dependent} lower bound that is sharper than the naive one mentioned above. We confirm in experiments that the naive FQE lower bound and the CFQE lower bound are in fact at different orders of magnitude.
\subsection{Model-Based Method For Stationary Transition Kernels}

While CFQE searches for the worst-possible \emph{policies}, we discuss a method here that searches for the worst possible \emph{transition dynamics} that are consistent with the data. Note that since $\pi_e$ is confounder-oblivious, the induced transitions $\prob^{\pi_e}_h(s' \mid s)$ are determined by the marginalized transition dynamics defined as $\prob_h(s' \mid s, a) := \sum_u P_h(u \mid s) \prob_h(s' \mid s, a, u)$. This is clear from the following computation: $
    \prob^{\pi_e}_h(s' \mid s) = \sum_{u,a} \pi_{e,h}(a \mid s) P_h(u \mid s) \prob_h(s' \mid s, a, u)
    = \sum_a \pi_{e,h}(a \mid s) \left(\sum_u P_h(u \mid s) \prob_h(s' \mid s, a, u)\right) = \sum_a \pi_{e,h}(a \mid s) \prob_h(s' \mid s,a)$.

We note that CFQE optimizes separately over the data at each timestep $h$. In particular, if the marginalized transition kernel were stationary, then the method would not leverage its stationarity. Our model-based method can leverage this, and we therefore assume the stationarity of transition dynamics and of $P(u \mid s)$ in this section. For ease of exposition, we also assume that $\pi_b$ and $\pi_e$ are stationary. The method can be modified to work for potentially time-dependent $\pi_b$ and $\pi_e$, which we do in Appendix~\ref{sec:mbproofs}.

We now describe the method. Let the empirically observed transitions be $\hat{\prob}^{\pi_b}(s' \mid s,a)$, and denote its value in the limit of infinite data by $\prob^{\pi_b}(s' \mid s,a)$. We know that the latter is stationary under our expository simplification. Let $\hat{\alpha}(s,a)$ and $\hat{\beta}(s,a)$ be obtained using the estimate $\hat{\pi}_b(s,a)$ Denote by $\mathcal{G}$ the set of marginalized transitions $\prob(s' | s,a)$ that fall between $\hat{\alpha}(s,a)(\hat{\prob}^{\pi_b}(s' | s,a))$ and $\hat{\beta(s,a)}(\hat{\prob}^{\pi_b}(s' | s,a))$ for each $s',a,s$. Our model-based method amounts to solving the following optimization problem:
\begin{gather} \label{opt}
	\min_{V_1(s_0), V_2, \ldots, V_H, V_{H+1} = 0,\prob} V_1(s_0) \\ 
	\;\; \text{s.t. } \prob\in\mathcal{G}, \;\;
	\sum_{s'}\prob(s'\mid s,a) = 1 \;\; \forall s,a.\notag \\
	V_h(s) = \pi_e(\cdot\mid s)^T(R_s + \prob_s V_{h+1}(\cdot)) \;\; \forall h \in \{1,...,H\}, s \notag
\end{gather}
where $V_{H+1} = 0$ and $\prob_s\in\mathbb{R}^{A\times S}$ is the matrix whose rows are $\prob(\cdot\mid s,a)$ for each $a$, $R_s\in\mathbb{R}^A$ and $V_{h+1}(\cdot)\in\mathbb{R}^S$. This corresponds to minimizing the value function $V_1(s_0)$ over the set $\mathcal{G}$ of state transition probabilities, using $H\cdot S$ Bellman backup constraints to encode the Bellman equation. 

While this method is similar to the model-based method in \cite{bruns2021model} inspired by robust MDP literature, it is important to note that unlike \cite{bruns2021model}, we look at uncertainty sets for each $s,a$ (instead of just one for each $s$) and make no additional assumption on model-sensitivity. In particular, model sensitivity and the uncertainty sets for the true marginalized transition kernel are completely determined by $\Gamma$. This method possesses several theoretical guarantees, proved in Appendix \ref{sec:mbproofs}.
\begin{restatable}[Error for the Model-Based Method]{thm}{errModelBased}
	\label{thm:model_base}
	Suppose $\Gamma = 1+\varepsilon$ in Assumption~\ref{assump:sensitivity}. Then the value estimation from solving~\eqref{opt} with infinite data, denoted by $\tilde{V}_1$, provides a lower bound no looser than CFQE and satisfies that $|V_1^{\pi_e}(s_0) - \tilde{V}_1(s_0)| = O(\varepsilon H^2)$ for any range of $\varepsilon$.
\end{restatable}
We will find in experiments that the lower bound produced by the model-based method is in fact tighter in some scenarios. In the finite-sample setting, we use point estimates $\hat{\prob}^{\pi_b}$ to construct $\mathcal{G}$. In another version for finite samples, one can account for estimation error of $\hat{\prob}^{\pi_b}$ by constructing a Hoeffding confidence interval for the state transition probabilities, and using it to construct $\mathcal{G}$ instead. We discuss this in Appendix \ref{sec:mbproofs}. Denoting the output of either version by $\hat{V}_1$, the theorem below guarantees that $\hat{V}_1$ is a consistent estimate for the infinite-sample lower bound $\tilde{V}_1$. We prove it in Appendix~\ref{sec:mbproofs}, and the Hausdorff-distance-based technique developed for the proof can be used to provide similar guarantees for FQE and CFQE.

\begin{restatable}[Consistent Estimation of the Lower Bound]{thm}{ModelBasedConsistent}\label{thm:model-based-consistent}
The estimated lower bound from the model-based method is strongly consistent for the lower bound $\tilde{V}_1$, where $\tilde{V}_1$ is the lower bound estimate of the value function from solving \eqref{opt} with infinite data. That is, $\hat{V_1} \overset{a.s.}{\to} \tilde{V}_1$.
\end{restatable}

\paragraph{A Computationally Efficient Method.}

Although the non-convex optimization problem in \eqref{opt} is solvable with off-the-shelf solvers, such problems can be difficult to solve efficiently. We provide a method, Algorithm~\ref{algo:projOPE}, in Appendix~\ref{sec:mb-varations} for quicker computation of lower bounds. This method approximately solves the model-based optimization problem in \eqref{opt} via projected gradient descent, optimizing over $\prob$ while maintaining the Bellman constraints.

\paragraph{Non-Stationary Model-Based Method.}

To handle non-stationary settings, we provide Algorithm~\ref{algo:MBRelax} in Appendix~\ref{sec:mb-varations}. This relaxes the Bellman backup constraints in \eqref{opt} by sequentially solving $H$ efficiently solvable quadratic programs. This is essentially the model-based analogue to CFQE.

\section{OPE under Confounders with Memory}\label{sec:history-dependent-confounders}
Sensitivity constraints do not alone contribute to the error upper bounds in Section~\ref{sec:history-ind-confounders} -- the memorylessness of confounders is an important ingredient. We demonstrate below that OPE under confounders with memory is hard even for $\pi_b$ with the best-case sensitivity, $\Gamma = 1$. Recall that $\Gamma=1$ corresponds to confounder-oblivious behavior policies. Specifically, the theorem below shows FQE and any method that lower bounds FQE will have $\Omega(H)$ worst-case error for confounders with memory, even for unconfounded $\pi_b$ and $\pi_e$ with bounded concentrability and given infinite data. We prove it in Appendix~\ref{sec:hist-dep-lower-bound}.

\begin{restatable}[Lower Bound for Confounders with Memory]{thm}{HistDependentLowerBound}\label{thm:hist-dependent-lower-bound}
There exists an MDP $\M$ having confounders with memory, a stationary unconfounded behavior policy $\pi_b$ with sensitivity $\Gamma = 1$, a stationary evaluation policy $\pi_e$ with $\frac{\pi_e(a \mid s)}{\pi_b(a \mid s)} \leq 2\ \forall s,a,$ and a state $s_1$, so that $V_1^{\pi_e}(s_1) = \Omega(H)$ while the output of FQE for $\pi_e$ is $O(\log H)$, even with infinite data.
\end{restatable}

While the challenges of FQE for POMDPs in general are qualitatively understood \cite{uehara2022futuredependent}, we show that it can be \emph{arbitrarily} bad even in the much milder setting of confounded MDPs with unconfounded $\pi_b$ and $\pi_e$. This suggests that making more specific assumptions about confounders with memory is necessary for designing OPE algorithms with theoretical guarantees. One example of such an assumption is the global confounder assumption, discussed below.

\subsection{Clustering-Based OPE for Global Confounders}

The main message of this section is that the dependence of confounders across timesteps can make it possible to pin down the effect of confounding and achieve consistent OPE, given enough structure to the dependence. We bring our focus to global confounders (Assumption~\ref{assump:global_u}) in the case where transition dynamics are stationary, and so are the behavior and evaluation policies. Notice that in the stationary setting, global confounders exactly describe a mixture of MDPs. Let the value of the evaluation policy $\pi_e$ under the dynamics induced by confounder $u$ be $V_1(s_0; u, \pi_e)$. If one can estimate this value and $P(u)$ for each $u$, then one can provide point estimates of the policy value $V_1^{\pi_e}(s_0) = \sum_u P(u) V_1(s_0; C_u, \pi_e)$.

We use Algorithm~\ref{algo:clusterOPE} as a broad meta-algorithm that takes a clustering algorithm and an OPE algorithm as input. We cluster the data and apply the OPE algorithm separately to each cluster to obtain a consistent final policy value estimate $\hat{V}_1(s_0; \pi_e)$. The crucial intuition behind this algorithm is the fact that the value estimate is a weighted average of value estimates over each confounder. 

\begin{algorithm}[h]
	\centering
	\caption{Clustering-Based OPE}
	\begin{algorithmic}[1]
		\STATE \textbf{input: } Number of clusters $U$, evaluation policy $\pi_e$, clustering algorithm \texttt{cluster()}, OPE estimator \texttt{ope()}.
		\STATE \textbf{run subroutine: } Use \texttt{cluster()} to obtain clusters $C_1,...,C_U$.
		\STATE  Obtain cluster weight estimates $\hat{P}(u) := \frac{|C_u|}{N_{traj}}$.
        \STATE \textbf{run subroutine: } Estimate $\hat{V}_1(s_0; C_u, \pi_e)$ for each cluster $C_u$ using \texttt{ope()}.
		\STATE \textbf{return: } Output the final policy value estimate $\hat{V}_1(s_0 ; \pi_e) = \sum_{u=1}^U \hat{P}(u_i) \hat{V}_1(s_0; C_u, \pi_e)$.
	\end{algorithmic}
\label{algo:clusterOPE}
\end{algorithm}

To present an end-to-end theoretical guarantee, we instantiate the meta-algorithm using the recent work of \cite{ambuj2022mixmdp} as our clustering algorithm and the data-splitting tabular-MIS (marginalized importance sampling) estimator from \cite{yin2020asymptotically} as our OPE estimator. To satisfy the assumptions of \cite{ambuj2022mixmdp} and \cite{yin2020asymptotically}, we require 3 additional assumptions, discussed in their papers.

\begin{assumption}[Mixing, from \cite{ambuj2022mixmdp}]\label{assump:mixing}
Let the $U$ Markov chains on $\mathcal{S} \times \mathcal{A}$ induced by the various behavior policies $\pi(a \mid s, u)$, each achieve mixing to a stationary distribution $d_{u}(s,a)$ with mixing time $t_{mix, u}$. Define the overall mixing time of the mixture of MDPs to be $t_{mix} := \max_u t_{mix, u}$.
\end{assumption}

\begin{assumption}[Model Separation, from \cite{ambuj2022mixmdp}]\label{assump:model_sep}
There exist $\alpha, \Delta > 0$ so that for each pair $u_1, u_2$ of confounders, there exists a state action pair $(s,a)$ (possibly depending on $u_1, u_2$) so that the stationary distributions under each confounder $d_{u_1}(s,a), d_{u_2}(s,a) \geq \alpha$ and $\|\prob^{(u_1)}(\cdot \mid s,a) - \prob^{(u_2)}(\cdot \mid s,a)\|_2 \geq \Delta$.
\end{assumption}

\begin{assumption}[Concentrability and Exploration, from \cite{yin2020asymptotically}]\label{assump:concentrability}
For $d_m := \min\{d^{\pi_b}_h(s) \mid d^{\pi_e}_h(s)>0\}$, $d_m > 0$, and there exist constants $\tau_a$ and $\tau_s$ so that for all $s, a, h$ $\frac{d^{\pi_e}_h(s)}{d^{\pi_b}_h(s)} \leq \tau_s$ and $\frac{\pi_e(a \mid s)}{\pi_b(a \mid s)} \leq \tau_a$.
\end{assumption}

We can therefore leverage the work of \cite{ambuj2022mixmdp} to achieve exact clustering with enough data under Assumptions \ref{assump:global_u}, \ref{assump:mixing}, and \ref{assump:model_sep}, recovering the unobserved global confounder $u_n$ in each trajectory up to permutation\footnote{They recover clusters, which is sufficient as we only need to know confounders up to renaming the labels.}. Then, when using the estimator from \cite{yin2020asymptotically} under Assumption~\ref{assump:concentrability}, we obtain the following guarantee.

\begin{restatable}[Sample Complexity for OPE under Global Confounding]{thm}{ClusterOPESampleComplexity} \label{thm:cluster-ope-sample-complexity}
Under Assumptions~\ref{assump:global_u}, \ref{assump:mixing}, \ref{assump:model_sep}, \ref{assump:concentrability}, there are constants $H_0$, $N_0$ depending polynomially on $\frac{1}{\alpha}, \Delta, \frac{1}{\min_u P(u)}, \log(1/\delta)$, so that for $n$ trajectories of length $H \geq H_0t_{mix}\log(n)$, we have that $|\hat{V}_1(s_0 ; \pi_e) - V_1(s_0 ; \pi_e)| < \epsilon$ with probability at least $1-\delta$ if $n \geq \Omega(\max(n_1, n_2, n_3, n_4))$, where
\begin{align*}
    n_1 &:= U^2SN_0\log(1/\delta), \hfill n_2 := \frac{\log(U/\delta)}{\min(\epsilon^2/H^2, \min_u P(u)^2)} \\
    n_3 &:= \frac{H^2\tau_a\tau_sSA \log(U/\delta)}{\epsilon^2}, n_4 := \frac{\tau_aH}{d_m}
\end{align*}
\end{restatable}

The first term represents the sample complexity for exact clustering (given in \cite{ambuj2022mixmdp}), the second term corresponds to estimating $P(u)$ accurately and the third and fourth come from the sample complexity of the OPE estimator (given in \cite{yin2020asymptotically}). In Appendix~\ref{sec:clustering-ope}, we prove a more general version of this theorem, where the OPE estimator makes an assumption $A(b)$ depending on a parameter vector $b$ and has sample complexity $N_2(\delta, \epsilon, b)$. Results analogous to Theorem~\ref{thm:cluster-ope-sample-complexity} can thus be produced using Corollary 1 of \cite{duan2020minimax}, or other off-policy estimators listed in section 2 of \cite{zhang2022opelist} viewed in a tabular setting. This is the first result that provides sample complexity guarantees for consistent point estimates under confounding. Theorem~\ref{thm:tmix-lower-bound} in Appendix~\ref{sec:tmix-lower-bound} shows that requiring that $H \geq \Omega(t_{mix})$ in Theorem~\ref{thm:cluster-ope-sample-complexity} is unavoidable, even for small $t_{mix} = O(\log(S))$.


\section{Policy Optimization under Confounding}\label{sec:optimization}
We first make an elementary observation that given a bound on the OPE error $|\hat{V}_1(\pi) - V_1(\pi)|$ and an optimizer for the value estimate $\hat{\pi}^* \in \argmax \hat{V}_1(\pi)$, we can obtain a sub-optimality bound for $\hat{\pi}^*$. We show this explicitly in Appendix~\ref{sec:iid_opt}, noting that this is agnostic to the existence and the nature of confounding. 

\paragraph{Policy Gradients on Lower Bounds under Memoryless Confounding.} Recall that in Section~\ref{sec:history-ind-confounders}, we produced lower bounds on the value function under memoryless confounding with a sensitivity model. In lieu of optimizing a point estimate of the policy's value, we can instead improve this lower bound.

Recall that Algorithm~\ref{algo:projOPE} in Appendix \ref{sec:mb-varations} computes a lower bound on $V_1(s_0)$ by projected gradient descent. We can backpropagate gradients relative to the evaluation policy, improving the lower bound on $V_1(s_0)$, and therefore the policy, with gradient ascent. We present the case with stationary transition structures in the max-min formulation below in the interest of lucidity, noting that it immediately generalizes to non-stationary transition structures as well.

\begin{equation}\label{equ:max-min-problem}
    \max_{\theta \in \Theta} \min_{\prob \in \mathcal{G}} V_1(s_0 ; \pi_\theta, \prob)
\end{equation}

We repeat the alternating process of finding $\prob \in \mathcal{G}$ to minimize $V_1(s_0)$ given an evaluation policy $\pi_\theta$ and then performing a gradient ascent update on $\pi_\theta$. This is illustrated fully in Algorithm~\ref{algo:polGrad} in Appendix \ref{sec:iid_opt},\footnote{Given libraries like \texttt{cvxpylayers}, we can also  perform gradient ascent on any lower bound from differentiable convex optimization. This includes the lower bounds generated by the relaxation of the model-based algorithm (Alg.~\ref{algo:MBRelax}) and CFQE (Alg.~\ref{algo:cFQE}). We state general lemmas that back our claims.} where we discuss local convergence guarantees for the method.

\paragraph{Policy Gradients under Global Confounding.} 

Recall that we hope to solve $\argmax_{\pi_e} V_1(s_0; \pi_e)$, where $V_1(s_0; \pi_e) = \sum_{u} P(u) V_1(s_0;u;\pi_e)$, for confounder-unaware evaluation policy $\pi_e$. This is the Weighted-Value Problem in \cite{steimle2021mmdp}, which is NP-hard according to Proposition 2 in their paper.

We discuss a policy gradient method for this problem. Let $ Z(\theta) := \nabla_\theta V_1(s_0 ; \pi_\theta)$. By Assumption~\ref{assump:global_u}, $Z(\theta) = \nabla_\theta \mathbb{E}_{u}[V_1(s_0 ; u, \pi_\theta)] = \nabla_\theta \sum_u P(u) V_1(s_0 ; u, \pi_\theta) = \sum_u P(u) \nabla_\theta V_1(s_0 ; u, \pi_\theta)$. Therefore, if we have gradient estimates $\hat{Z}_i(\theta)$ of $Z_i(\theta) = \nabla_\theta V_1(s_0 ; u_i, \pi_\theta)$ for each cluster, we can obtain the final policy gradient estimate as a weighted sum, given by $\hat{Z}(\theta) = \sum_{u=1}^U \hat{P}(u_i) \hat{Z}_i(\theta)$. We present this as Algorithm~\ref{algo:clusterOpt} in Appendix~\ref{sec:clustering-pg}.

We then perform standard gradient descent for $T$ iterations on the policy parameters $\theta$, with the update rule given by $\theta_{t+1} = \theta_t - \eta \hat{Z}(\theta_t)$. In analyzing this procedure, we instantiate $\hat{Z}_i$ using the (statistically) Efficient Off-Policy Policy Gradient (EOPPG) estimator from \cite{kallus2020statistically}, which enjoys an $\Theta(H^4/n)$ MSE guarantee instead of the $2^{\Theta(H)}\Theta(1/n)$ worst-case sample complexity of REINFORCE \cite{kallus2020statistically}. We assume that the gradient of $V_1$ is bounded by $L$, which holds if $V_1$ is $L$-Lipschitz. Additionally, let assumptions for Theorem 12 in \cite{kallus2020statistically} hold. We obtain a bound on the norm of the policy gradient that shows convergence to a stationary point. Theorem~\ref{thm:clustering-pg-endtoend} below holds when $H \geq H_0t_{mix}\log n$, for $H_0, N_0$ as in Theorem~\ref{thm:cluster-ope-sample-complexity}. It is proved in Appendix~\ref{sec:clustering-pg}. 

\begin{restatable}{thm}{ClusteringPG} \label{thm:clustering-pg-endtoend}
Let us have large enough $\beta > 1$ and $T=n^\beta$, for $n \geq \Omega\left(\max\left(U^2SN_0\log(1/\delta), \frac{\log(U/\delta)}{\min_u P(u)^2}\right)\right)$. $\frac{1}{T} \sum_{t=1}^T ||\nabla_\theta V_1(s_0;\pi_{\theta_t})||^2 = O(\max(\epsilon_{MSE}, \epsilon_{freq})$, where $\epsilon_{MSE} = \frac{H^4 \log (n U/\delta)}{n \min_u P(u)}$, and $\epsilon_{freq} = \frac{L^2\log(U/\delta)}{n}$
\end{restatable}

\section{Numerical Experiments}\label{sec:experiments}

\paragraph{Gridworld for Memoryless Confounders.}

\begin{figure}[h]
    \centering
    \includegraphics[width=7cm, height=5cm]{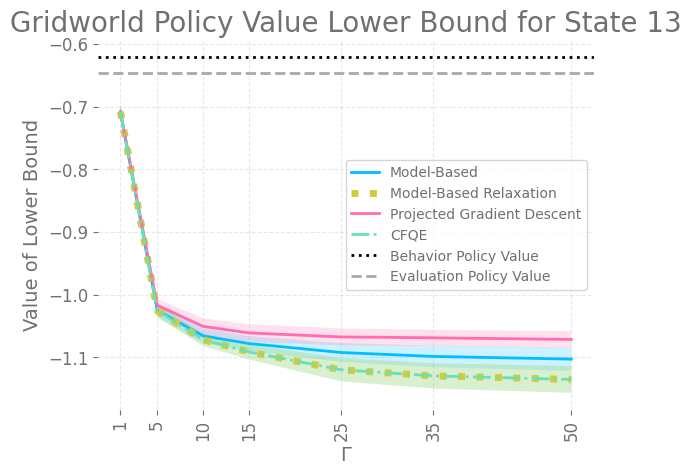}
    \caption{OPE for Memoryless Confounders. Comparison of our model-based method, its non-stationary relaxation (Alg.~\ref{algo:MBRelax}), its projected gradient descent variant (Alg.~\ref{algo:projOPE}), and CFQE on state 13 in a 16-state gridworld. Confidence intervals (CIs) are one standard deviation wide and computed over 30 trials. $H=8$.}
    \label{fig:iidcomparison}
\end{figure}

We examine the performance of the methods in Section~\ref{sec:history-ind-confounders} on the 4x4 gridworld environment used by \cite{bruns2021model}, with i.i.d. (and thus memoryless) confounders. We implement the model-based method and its variations using the point estimates $\hat{\prob}^{\pi_b}$ instead of Hoeffding confidence intervals for ${\prob}^{\pi_b}$, for a fair comparison with CFQE. The horizon is $H = 8$, and $\Gamma$ ranges from $1$ to $50$. We plot the policy values against $\Gamma$ in Figure~\ref{fig:iidcomparison}. Across all 16 states, the model-based method's lower bound is always either as good as or tighter than that of CFQE, but the gap in performance is seen most starkly in state 13 (which we display in Figure \ref{fig:iidcomparison}). The output of FQE is obtained at $\Gamma = 1$ and is at most $-0.7$. By the remark after the proof of Theorem~\ref{thm:err_FQE}, the naive lower bound obtained using FQE is less than $-0.7 - \frac{\varepsilon H^2}{2} = -0.7 - 32\varepsilon$. This is quite literally "off-the-chart" here, showing that using FQE for lower bounds would be ineffective in practice.

We also study policy improvement. Figure \ref{fig:dynamics} displays the training dynamics and convergence of Algorithm \ref{algo:polGrad}, where we perform gradient ascent on a lower bound obtained by Algorithm \ref{algo:projOPE}. We visualize the learned policy, which is appropriately conservative: on a horizon of 8, the agent will likely not reach the goal state from the first few states and move to the top left corner appropriately. Finally, we plot the increase in the lower bound on policy value against progressing gradient ascent iterations, starting at $\pi_e$. Note that even our lower bounds all eventually exceed the true (ground truth) values of $\pi_b$ and $\pi_e$, displaying improvement.

\begin{figure}[h]
    \centering
    \includegraphics[width=5.2cm, height=4cm]{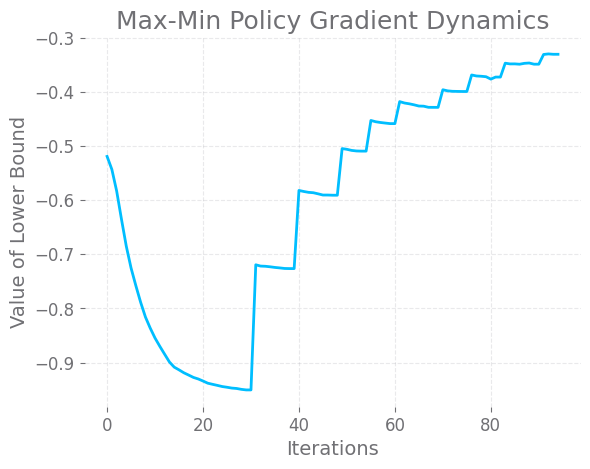}
    \includegraphics[width=2cm, height=4cm]{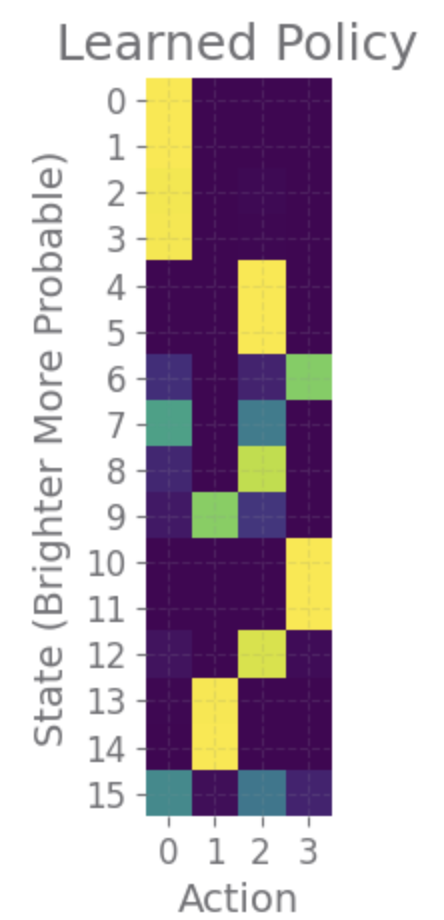}
    \includegraphics[width=7cm, height=5cm]{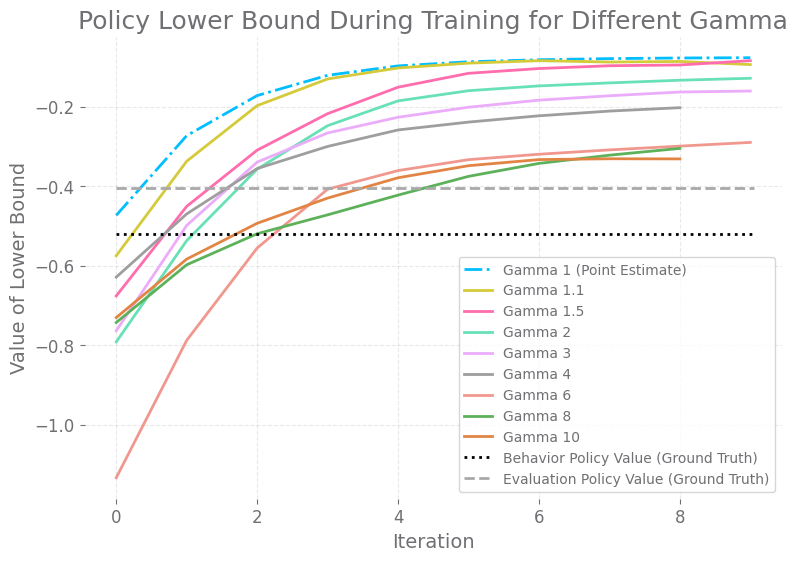}
    \caption{Policy Improvement for Memoryless Confounders. Top Left: Loss curve dynamics of max-min gradient descent. Top Right: Resulting policy $\hat{\pi}^*$ for $\Gamma = 10$ in 4x4 gridworld with actions indexed by WENS. Brighter colors indicate higher $\hat{\pi}^*(a\mid s)$. Bottom: Increase in the lower bound on $V_1^{\pi_\theta}$ as gradient ascent iterations progress. $H=8$.}
    \label{fig:dynamics}
\end{figure}



\paragraph{Sepsis Simulator for Global Confounders.}

We examine the performance of the method of Algorithm \ref{algo:clusterOPE} on the sepsis simulator of \cite{oberst2019counterfactual}, especially in terms of the choice of the clustering algorithm. Once we hide the diabetes status of each patient, it becomes a global confounder. 
The confounder-aware behavior policy is the same behavior policy in \cite{oberst2019counterfactual}, and the evaluation policy is $\pi_e := \frac{1}{U}\sum_u \pi_b(a|s,u)$. In the simulator, glucose levels are generated i.i.d, with their distribution determined by the presence or absence of diabetes. This makes them easy proxies for diabetes, so we hide glucose levels during the clustering phase to make the clustering problem harder. 

On the top left of Figure~\ref{fig:sepsis}, we compare the clustering error for the method of \cite{ambuj2022mixmdp} with that of classical soft EM with random initialization. In the top right, we plot a measure of the relative error in OPE against trajectory length. The relative error is computed as $\frac{\max_s |\hat{V}_1^{\pi_e}(s) - {V}_1^{\pi_e}(s)|}{\max_s |{V}_1^{\pi_e}(s)|}$. The plot compares the performance of Algorithm~\ref{algo:clusterOPE} instantiated with FQE coupled with either soft EM with random initialization or the method of \cite{ambuj2022mixmdp}. At the bottom, we show the convergence of Algorithm \ref{algo:clusterOpt}, instantiated using the off-policy policy gradient variant that \cite{kallus2020statistically} attributes to \cite{degris2013offpolicy}. We compare the same possibilities for clustering as above. We observe that in general, the method of \cite{ambuj2022mixmdp} outperforms randomly initialized soft EM, allowing for both OPE and policy improvement. Our experimental results highlight the effectiveness of our method as well as the importance of the clustering algorithm. 


\begin{figure}[h!]
    \centering
    \includegraphics[width=3.5cm, height=4.5cm]{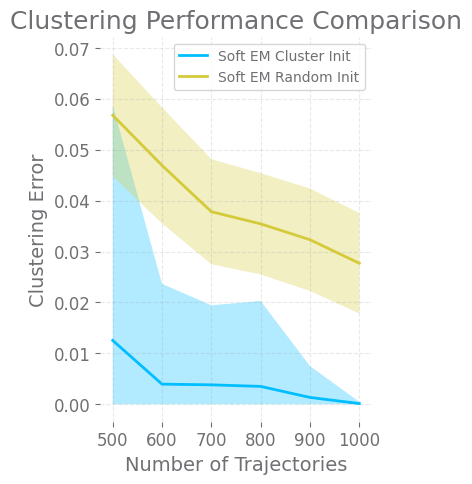}
    \includegraphics[width=3.5cm, height=4.5cm]{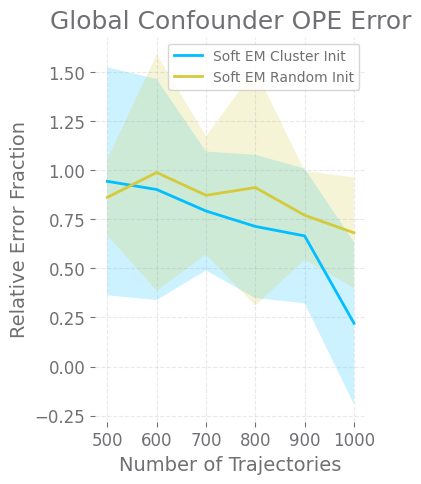}
    \includegraphics[width=5cm, height=4cm]{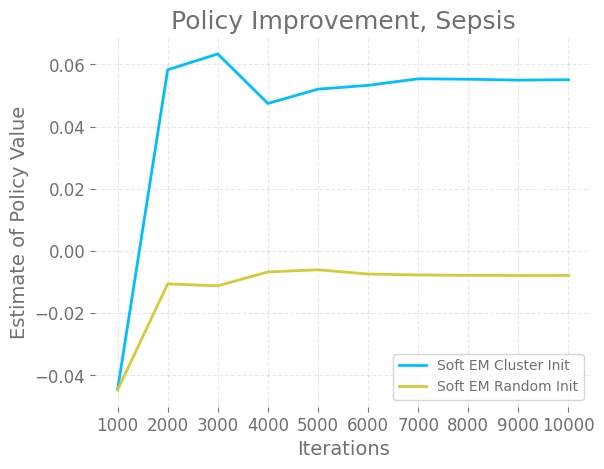}
    \caption{Top Left: Average performance of the clustering method from Kausik et al. Top Right: Average relative error of clustering-based OPE with different clustering algorithms. Bottom: Improvement in estimates of policy values under gradient ascent coupled with different clustering algorithms, see Appendix~\ref{sec:experimental-details} for details. We average over 30 trials, confidence intervals are 1 standard deviation wide. $H=60$.}
    \label{fig:sepsis}
\end{figure}

\section{Conclusion and Future Work}

We have provided a broad, structured view of the landscape of confounded MDPs, studying the OPE and OPI problems under various confounding assumptions. The paper has discussed existing methods, presented new ones and provided theoretical and empirical grounding for the methods. We hope that the insights here will springboard further work on confounded MDPs. 
In particular, while we address the sensitivity assumption, a big-picture view of other assumptions like bridge functions and instrumental variables is needed. For general confounders with memory, note that while Theorem~\ref{thm:hist-dependent-lower-bound} rules out FQE and related methods, other methods must be explored. There are also specific structures on confounders with memory, besides global confounders, that can be formulated and studied. Finally, many of our methods (such as the gradient-based methods presented) can be extended to handle continuous state spaces via function approximation. \cite{Shi2021AML} provide methods under assumptions on the existence and learnability of bridge functions, being one of the first works to address this. However, work on confounding with continuous state and action spaces is still relatively sparse, and is an exciting setting to explore. 

%

\newpage
\bibliographystyle{apalike}
\bibliography{ref}

\newpage

\appendix
\onecolumn


\section{Experimental Details}
\label{sec:experimental-details}

\paragraph{Computing Infrastructure.} All numerical experiments were run on a single desktop computer with an Intel i9-13900K CPU, 128 gigabytes of RAM, and an NVIDIA RTX 3090 graphics card.

\paragraph{Estimating Policy Values for Global Confounders.} Due to computationally expensive operations needed to compute the exact policy value for confounders, we use estimates of the policy values instead. Namely, we get estimates $\hat{V}_1(s_0, u, \pi)$ for a policy $\pi$, and report $\sum_u P(u)\hat{V}_1(s_0, u, \pi)$. Computing the true values $V_1(s_0, u, \pi)$ is computationally far more expensive. The estimates $\hat{V}_1(s_0, u, \pi)$ are obtained using standard FQE applied to the standard, unconfounded MDP determined by confounder $u$.

\section{Lower Bounds for Memoryless Confounders}\label{sec:iid-lower-bounds}
We recall and prove Theorem~\ref{thm:iid-lower-bound}.

\iidLowerBound*

\begin{proof} Consider two confounded MDP environments $\M_1$ and $\M_2$. 

\paragraph{Environments.}

In both environments:
\begin{itemize}
	\item $\mathcal{S} = \{1,2\}$, $\U = \{1,2\}$, $\A = \{1,2\}$, horizon $H$.
	\item $r(s = 1) = 1$, $r(s = 2) = 0$.
\end{itemize}

For confounders:
\begin{itemize}
	\item $P_1(u = 1) = \frac{1}{2} - \varepsilon, P_1(u = 2) = \frac{1}{2} + \varepsilon$.
	\item $P_2(u = 1) = \frac{1}{2} + \varepsilon, P_2(u = 2) = \frac{1}{2} - \varepsilon$.
\end{itemize}

For full state transitions:
\begin{align*}
	&\prob_1(s'=1\mid s,u = 1,a = 1) = z, 	\prob_1(s'=1\mid s,u = 2,a = 1) = 1-z\\
	&\prob_1(s'=1\mid s,u = 1,a = 2) = z_1, 	\prob_1(s'=1\mid s,u = 2,a = 2) = z_2\\
	&\prob_2(s'=1\mid s,u = 1,a = 1) = z, 	\prob_2(s'=1\mid s,u = 2,a = 1) = 1-z\\
	&\prob_2(s'=1\mid s,u = 1,a = 2) = z_2, 	\prob_2(s'=1\mid s,u = 2,a = 2) = z_1
\end{align*}

Next, consider two behavior policies $\pi_{b_1}$ and $\pi_{b_2}$:
\begin{align*}
    &\pi_{b_1}(a = 1\mid s, u = 1) = \frac{1}{2} + \varepsilon, \ \ \pi_{b_1}(a = 1\mid s, u = 2) = \frac{1}{2} - \varepsilon\\
	&\pi_{b_2}(a = 1\mid s, u = 1) = \frac{1}{2} - \varepsilon, \ \ \pi_{b_2}(a = 1\mid s, u = 2) = \frac{1}{2} + \varepsilon
\end{align*}

And an evaluation policy $\pi_e$:
\begin{align*}
	\pi_e(s) = 1, \ \ \text{for } s = \{1,2\}.
\end{align*}
\paragraph{Data Collection. } Suppose we collect data using $\pi_{b_1}$ in $\M_1$ and using $\pi_{b_2}$ in $\M_2$. Notice that the sensitivity $\Gamma$ is given by
$$\Gamma = \left( \frac{\frac{1}{2}+\varepsilon}{\frac{1}{2}-\varepsilon}\right) \left(\frac{\frac{1}{2}+\varepsilon^2}{\frac{1}{2}-\varepsilon^2}\right)$$

\paragraph{Observations. } Note that in the limit, i.e. infinite data, the observed transition probabilities and policies are given by 
\begin{align*}
	\hat{\prob}_1(s',a\mid s) &= \sum_u P_1(u)\pi_{b_1}(a\mid s,u)\prob_1(s'\mid s,u,a),\\
	\pi_1(a\mid s) &= \sum_u P(u)\pi_1(a\mid s,u),\\
	\hat{\prob}_1(s'\mid s,a) &= \hat{\prob}_1(s',a\mid s) / \pi_1(a\mid s).
\end{align*}

One can then easily verify that for all $s,a,s'$, the observed transition probabilities will be equal:
\begin{align*}
	\hat{\prob}_1(s',a\mid s) = \hat{\prob}_2(s',a\mid s),
\end{align*}

For example, $\hat{\prob}_i(s'=1,a=1\mid s) = x(1-x)$ for $i=1,2$.

The state transition and the observed policy induced by the two policies in their corresponding environment are thus also equal:
\begin{align*}
	\pi_1(a\mid s) &= \pi_2(a\mid s),\\
	\hat{\prob}_1(s'\mid s,a) &= \hat{\prob}_2(s'\mid s,a).
\end{align*}

That means, no algorithm can distinguish the two environments based on the given two datasets.

\paragraph{Value under the evaluation policy. }
 Recall that at each step, we take action $1$. Note that the true marginalized state transitions will be different, which are what a confounder-oblivious policy will interact with:
\begin{align*}
	&\prob_1(s' = 1\mid s, a = 1) = \sum_u P_1(u)\prob_1(s' = 1\mid s, u, a = 1) = \left(\frac{1}{2}  + \varepsilon\right)(1-z) + \left(\frac{1}{2}  - \varepsilon\right)z\\
	&\prob_2(s' = 1\mid s, a = 1) = \sum_u P_2(u)\prob_2(s' = 1\mid s, u, a = 1) = \left(\frac{1}{2}  - \varepsilon\right)(1-z) + \left(\frac{1}{2}  + \varepsilon\right)z
\end{align*}

Note that $\prob^{\pi_e}_{i}(s'=1 \mid s) = \prob_i(s'=1 \mid s, a=1)$. Since state transitions are independent of the initial state, this is the same as generating a state independently at each step based on the action taken. Then under the evaluation policy $\pi_e(a=1 \mid s) =1$, the state $s=1$ is generated i.i.d. at each step with probability $p_i = \prob_i(s'=1 \mid s, a=1)$ in $\M_i$, while $s=2$ is generated with probability $1-p_i$. So, the reward of a trajectory is distributed according to $Bin(H, p_i)$, having an expected value of $V_1^{\pi_e}(\M_i) = Hp_i = H\prob_i(s'=1 \mid s, a=1)$.

\paragraph{Necessity of a Sensitivity Assumption}
Let $\varepsilon = \frac{1}{2}-\frac{1}{H^2}$, $z = 0$. We then have the following
\begin{align*}
	&V_1^{\pi_e}(\M_1) = H((1-\frac{1}{H^2})^2 + 1/H^4) = O(H)\\
	&V_1^{\pi_e}(\M_2) = 2H\cdot\frac{1}{H^2}(1-\frac{1}{H^2}) = O(\frac{1}{H}).
\end{align*}
From this example, we see that without information about $\Gamma$, no algorithm can universally give meaningful lower bounds for the true value function. One can compute that in this example, $\Gamma = \Theta(H^2)$.

\paragraph{Lower Bound on Value Estimation Under Sensitivity}
Let $\varepsilon$ be small and let $z = 0$. We then have the following.
\begin{align*}
	&V_1^{\pi_e}(\M_1) = H(\frac{1}{2} - \varepsilon)\\
	&V_1^{\pi_e}(\M_2) = H(\frac{1}{2} + \varepsilon)
\end{align*}

Note that $\Gamma = 1 + O(\varepsilon + \varepsilon^2) = 1 + O(\varepsilon)$ for small $\varepsilon$. Since any estimator will return the same value for both MDPs (because they are observationally indistinguishable under the behavior policy), any estimator will have a worst-case error of at least $\varepsilon H$. Thus, there does not exist a consistent estimator whenever $\Gamma > 1$.

\end{proof}
\newpage 

\section{FQE and Confounded FQE}
\label{sec:FQE_cFQE}
We describe the FQE and CFQE algorithms here, adapted for memoryless systems instead of merely stationary ones.

\subsection{FQE Algorithm}
\begin{algorithm}[h]
	\centering
	\caption{FQE}
	\begin{algorithmic}[1]
		\STATE \textbf{input: } evaluation policy $\pi_e$.
		\STATE \textbf{initialize:} $\hat{f}_{H+1}\leftarrow 0$.
		\FOR{$h = H, H-1,\ldots,1$}
		\item $\hat{f}_h(s,a)\leftarrow\E_{(s,a,s') \sim \mathcal{D}_{\pi_b}, h}\left[r_h(s,a)+\sum_{a'}\pi_{e, (h+1)}(a'\mid s')\hat{f}_{h+1}(s',a')\right], \forall s,a$.
		\ENDFOR
		\STATE \textbf{return: } $\sum_a \pi_{e,1}(a\mid s)\hat{f}_1(s,a)$ for $\forall s$.
	\end{algorithmic}
\end{algorithm}

\subsection{Confounded FQE Algorithm}

Confounded FQE (CFQE), proposed by \cite{bruns2021model}, provides an estimate for a lower bound by taking the characteristics of the data into account. Given infinite samples, this will actually be a lower bound, unlike the case of FQE. In particular, CFQE obtains an estimate for a lower bound by sequentially searching over the worst behavior policy consistent with the observations. 

Let $\hat{\pi}_{b, h}(a \mid s)$ and $\hat{\prob}_h(s'\mid s,a)$ be empirical estimates from finite data $\mathcal{D}_{\pi_b, h}$. Let $\prob^{\pi_b}_h(s'\mid s,a)$ be the limit of $\hat{\prob}_h(s' \mid  s,a)$ under infinite data. We then define the following uncertainty sets.

\begin{defini}[Valid Behavior Policy Set]
	\label{lemma:consistency}
	Under a memoryless confounder, for all $s,a,s'$, define $\mathcal{B}_{sa, h}$ to be the set of all $\pi(a\mid s,\cdot)$ that satisfy Assumption~\ref{assump:sensitivity} and the two equations below.
	\begin{gather*}
		\sum_{u\in\U} P_h(u \mid s)\pi_{b, h}(a\mid s,u) = \pi_{b, h}(a\mid s)\\
		\sum_{u\in\U}P_h(u \mid s)\pi_{b,h}(a\mid s,u)P(s'\mid s,u,a) = \pi_{b,h}(a\mid s)\prob^{\pi_b}_h(s'\mid s,a).
	\end{gather*}
	
	Now we define the following quantity using the posteriors $P_h^{\pi_b}(u \mid s,a)$, a confounded analog to inverse propensity weights.
	\begin{align*}
	    g_h(s,a,s') &:= \sum_u \left(\frac{P_h^{\pi_b}(u \mid s,a) \prob_h(s' \mid s,a,u)}{\hat{\prob}^{\pi_b}_h(s' \mid s,a)}\right) \frac{1}{\pi_{b,h}(a \mid s,u)}\\
	    &= \sum_u \left(\frac{P_h(u \mid s) \prob_h(s' \mid s,a,u)}{\hat{\prob}^{\pi_b}_h(s' \mid s,a)}\right) \frac{1}{\pi_{b,h}(a \mid s)}
	\end{align*}
	
    Theorem 1 and the discussion following that in \cite{bruns2021model} shows that we can reflect the same uncertainty using the set $\tilde{\B}_{sa, h}$ of possible values of $g_h(s,a, \cdot)$.
	\begin{align*}
	    \tilde{\B}_{sa, h} &:= \{ g_h(s,a,\cdot) \; \mid  \; \alpha_h(s,a)\leq\pi_{b,h}(a \mid  s)g_h(s, a, s')\leq \beta_h(s,a), \\ 
	& \qquad \qquad \sum_{s'} \pi_{b,h}(a \mid  s)g_h(s,a,s')\prob^{\pi_b}_h(s' \mid  s,a) = 1 \}
	\numberthis \label{eqn:tilde-b-sa-h}
	\end{align*}
\end{defini}

$\tilde{\B}_{sa,h}$ presents a reparameterization of the uncertainty that allows us to get rid of the explicit presence of the unknown variable $u$ while optimizing over the uncertainty set. Let $\hat{\B}_{sa, h}$ and $\hat{\tilde{\B}}_{sa, h}$ be the version of these sets determined by the point estimates $\hat{\pi}_{b}$ and $\hat{\prob}(s'\mid s,a)$ under finite data, instead of by their true values.

\begin{algorithm}[h]
	\centering
	\caption{Confounded FQE (adapted from~\cite{bruns2021model})}
	\begin{algorithmic}[1]
		\STATE \textbf{input: } evaluation policy $\pi_e$.
		\STATE \textbf{initialize:} $\hat{f}_{H+1}\leftarrow 0$.
		\FOR{$h = H, H-1,\ldots,1$}
		\STATE Compute
		\begin{align*}
		&\hat{f}_h(s,a) 
        := \\ &\min_{g_h(s,a,\cdot)\in\hat{\tilde{\B}}_{sa, h}} \E_{(s,a,s') \sim \mathcal{D}_{\pi_b, h}}\left[\hat{\pi}_{b,h}(a\mid s)g_h(s,a,s')
		\left(r_h(s,a)+\sum_{a'}\pi_{e,h}(a'\mid s')\hat{f}_{h+1}(s',a')\right) \right]
		\end{align*}
		\ENDFOR
		\STATE \textbf{return: } $\sum_a\pi_e(a\mid s)\hat{f}_1(s,a)$ for $\forall s$.
	\end{algorithmic}
\label{algo:cFQE}
\end{algorithm}






However, if a very poor estimate of $\hat{\pi_b}$ and $\hat{\prob}_{\pi_b}(s'\mid s,a)$ is collected (due to low $N(s,a)$ and/or $N(s)$), the estimated lower bound will be a lower bound on the output of FQE but not on the true value. To get a lower bound on the true value with probability at least $1-\delta$, we modify $\hat{\tilde{B}}_{sa,h}$ using error bounds $\err_\pi(N(s))$ and $\err_\prob(N(s,a))$ obtained using the Hoeffding inequality to get the following set.

\begin{align*}
    \{ g_h(s,a,\cdot) \; \mid & \; \alpha_h(s,a)\leq\pi_{b,h}(a \mid  s)g_h(s, a, s')\leq \beta_h(s,a), \\ 
	& \sum_{s'} \pi_{b,h}(a \mid  s)g_h(s,a,s')\prob^{\pi_b}_h(s' \mid  s,a) = 1\\ 
	& | \pi_{b,h}(s,a) - \hat{\pi}_{b,h}(s,a) |  \leq \err_\pi(N(s)),\\
	& | \prob^{\pi_b}_h(s' \mid  s,a) - \hat{\prob}_h(s' \mid s,a)| \leq \err_\prob(N(s,a)) \}
\end{align*}

Additionally, the observant reader will note that CFQE finds a different optimal $g_h$ for each time step. That is, it finds $H$ different functions $g_1(s,a,\cdot),...,g_H(s,a,\cdot) \in \tilde{\mathcal{B}}_{sa}$. If the transition structures were stationary, this does not leverage the stationarity. In that case, it is advisable to use our model-based method and its projected gradient descent version, as discussed in Section~\ref{sec:history-ind-confounders}.

\newpage

\section{FQE and CFQE Theoretical Results} \label{sec:proofs_FQE_cFQE}

\subsection{Proof of FQE Error Bounds, Theorem \ref{thm:err_FQE}}
We recall the theorem below.

\errFQE*
\begin{proof}
In the limit of an infinite amount of data, at every step of FQE, the update evaluates $\hat{f}_h(s,a)$ using:
\begin{align*}
	 \hat{f}_h(s,a) &=\argmin_{f_h(s,a)}\E_{(s,a,s') \sim \D_{\pi_b}^h}\left[\loss_{FQE}(f_h(s,a), s')\right]\\
	 &= \argmin_{f_h(s,a)}\sum_{u, s'}\prob^{\pi_b}(s', u \mid s,a) \loss_{FQE}(f_h(s,a), s')\\
	 &= \argmin_{f_h(s,a)}\sum_{u, s'}P_h^{\pi_b}(u \mid s, a) \prob_h(s'\mid s, u, a) \loss_{FQE}(f_h(s,a), s')\\
	&= \argmin_{f_h(s,a)}\sum_{u, s'}P_h(u \mid s)\frac{\pi_{b,h}(a\mid s,u)}{\pi_{b,h}(a\mid s)}\sum_{s'}\prob_h(s'\mid s, u, a) \loss_{FQE}(f_h(s,a), s')\\
\end{align*}
where $P_h^{\pi_b}(u \mid s,a)$ is the posterior on $u$ under $\pi_b$ and $$\loss_{FQE}(f_h(s,a), s') = \left(f_h(s,a) - r(s,a) - \sum_{a'}\pi_{e, h+1}(a'\mid s')\hat{f}_{h+1}(s', a')\right)^2$$
$\hat{f}_h(s,a)$ is then given by the following expression.
\begin{align*}
&\sum_{u, s'}P_h(u \mid s)\frac{\pi_{b, h}(a\mid s,u)}{\pi_{b,h}(a\mid s)}\prob_h(s'\mid s, u, a)
\left(r(s,a) + \sum_{a'}\pi_{e,h+1}(a'\mid s')\hat{f}_{h+1}(s', a')\right)\\
& = r(s,a) + \sum_{u,s'} P_h(u \mid s)\frac{\pi_{b,h}(a\mid s,u)}{\pi_{b,h}(a\mid s)}\prob_h(s'\mid s,u,a)\sum_{a'}\pi_{e, h+1}(a'\mid s')\hat{f}_{h+1}(s', a') 
\end{align*}

\paragraph{True marginalized transition structure.} Note that under any confounding-unaware policy $\pi_e$, the induced transition structure $\prob^{\pi_e}_h(s' \mid s)$ is determined by the marginalized transition dynamics $\prob_h(s' \mid s, a) := \sum_u P_h(u \mid s) \prob_h(s' \mid s, a, u)$. This is clear from the computation below.
\begin{align*}
    \prob^{\pi_e}_h(s' \mid s) &= \sum_{u,a} \pi_{e,h}(a \mid s) P_h(u \mid s) \prob_h(s' \mid s, a, u)\\
    &= \sum_a \pi_{e,h}(a \mid s) \left(\sum_u P_h(u \mid s) \prob_h(s' \mid s, a, u)\right) &= \sum_a \pi_{e,h}(a \mid s) \prob_h(s' \mid s,a)
\end{align*}

\paragraph{Bounding $\hat{f}_h(s,a)$.} By Assumption~\ref{assump:sensitivity} and the computations above, we can bound $\hat{f}_h(s,a)$ by:
\begin{align*}
	\hat{f}_h(s,a) &\leq r(s,a) + \frac{1}{\alpha_h(s,a)}\sum_{s'}\prob_h(s'\mid s,a)\sum_{a'}\pi_{e,h+1}(a'\mid s')\hat{f}_{h+1}(s',a'),\\
	\hat{f}_h(s,a) &\geq r(s,a) + \frac{1}{\beta_h(s,a)}\sum_{s'}\prob_h(s'\mid s,a)\sum_{a'}\pi_{e, h+1}(a'\mid s')\hat{f}_{h+1}(s',a').
\end{align*}

The ultimate goal is to bound $V_1^{\pi_e}(s) - \sum_a\pi_{e,1}(a \mid s)\hat{f}_1(s,a)$, which is given by $\sum_a\pi_{e,1}(a\mid s)\left(Q_1^{\pi_e}(s,a) - \hat{f}_1(s,a)\right)$. So, we consider the error of $\hat{f}_h(s,a)$ at every step, given by $\err_h(s,a) := Q_h^{\pi_e}(s,a) - \hat{f}_h(s,a)$. We will use the following relation.
\begin{align*}
    Q^{\pi_e}_h(s,a) &= r(s,a) + \sum_{u, s'} P_{h}(u \mid s)\prob_h(s' \mid s,a,u)V^{\pi_e}_{h+1}(s')\\
    &= r(s,a) + \sum_{s'} \prob_h(s' \mid s,a)V^{\pi_e}_{h+1}(s') \numberthis \label{eqn:q-next-term}
\end{align*}

At $h = H$, by definition
\begin{align*}
	\hat{f}_{H}(s,a) = r(s,a) = Q_H^{\pi_e}(s,a).
\end{align*}
Thus, we get that $\err_H(s,a) = 0$ for all $s,a$. Let $\beta_{max} := \max_{s,a,h} \beta_h(s,a)$ and let $\alpha_{min} = \min_{s,a,h} \alpha_h(s,a)$.

For step $H-1$,
\begin{align*}
	\err_{H-1}(s,a) &\leq \sum_{s'}\prob_{H-1}(s'\mid s,a)V_H^{\pi_e}(s') - \frac{1}{\beta_H(s,a)}\sum_{s'}\prob_{H-1}(s'\mid s,a)\sum_{a'}\pi_{e, H}(a'\mid s')\hat{f}_H(s',a')\\
	& = (1-\frac{1}{\beta_H(s,a)})\sum_{s'}\prob_{H-1}(s'\mid s,a)V_H^{\pi_e}(s')\\
	& \leq \left(1-\frac{1}{\beta_{max}}\right)  \sum_{s'}\prob_{H-1}(s'\mid s,a)  \left(1-\frac{1}{\beta_{max}}\right)
\end{align*}

By induction, we will show that for all $h$, the following holds.
\begin{align*}
	\err_h(s,a) \leq H-h - \sum_{i=1}^{H-h} \frac{1}{\beta_{max}^i}
\end{align*}

We know this for $h= H-1$. For the induction step, we show this for $h-1$ given the statement for $h$ using the following computation.
\begin{align*}
	\err_{h-1} &\leq \sum_{s'}\prob_{h-1}(s'\mid s,a)V_{h}^{\pi_e}(s') - \frac{1}{\beta_h(s,a)}\sum_{s'}\prob_{h-1}(s'\mid s,a)\sum_{a'}\pi_{e, h}(a'\mid s')\hat{f}_h(s',a')\\
	&\leq \sum_{s'}\prob_{h-1}(s'\mid s,a)V_h^{\pi_e}(s')\\
	& \qquad +\frac{1}{\beta_h(s,a)}\sum_{s'}\prob_{h-1}(s'\mid s,a)\sum_{a'}\pi_{e, h}(a'\mid s')(\err_h(s,a) - Q^{\pi_e}_h(s,a))\\
	& = (1-\frac{1}{\beta_h(s,a)})\sum_{s'}\prob_{h-1}(s'\mid s,a)V_h^{\pi_e}(s') + \frac{1}{\beta_h(s,a)}\err_h(s,a)\\
	& \leq \left(1-\frac{1}{\beta_{max}}\right)\sum_{s'}\prob_{h-1}(s'\mid s,a)(H-h+1) + +\frac{1}{\beta_h(s,a)}\err_h(s,a)\\
	&\leq \left(1-\frac{1}{\beta_{max}}\right)(H-h+1) + \frac{1}{\beta_{max}}\left(H-h - \sum_{i=1}^{H-h} \frac{1}{\beta_{max}^i}\right)\\
	&= H-h+1 - \sum_{i=1}^{H-h+1} \frac{1}{\beta_{max}^i}
\end{align*}

Thus, the result holds by induction, giving us the following final bound.
\begin{align*}
    Q_1^{\pi_e}(s,a) - \hat{f}_1(s,a) \leq H - 1 - \sum_{i=1}^{H-1} \frac{1}{\beta_{max}^i} = H - \frac{1 - \frac{1}{\beta_{max}^H}}{1 - \frac{1}{\beta_{max}}}
\end{align*}
Similarly, we have the lower bound below:
\begin{align*}
	Q_1^{\pi_e}(s,a) - \hat{f}_1(s,a) \geq H- 1 - \sum_{i=1}^{H-1} \frac{1}{\alpha_{min}^i} = H - \frac{1 - \frac{1}{\alpha_{min}^H}}{1 - \frac{1}{\alpha_{min}}}
\end{align*}

Recall that $\alpha_h(s,a) = \pi_{b,h}(a\mid s) + \frac{1}{\Gamma}(1-\pi_{b,h}(a\mid s))$ and $\beta_h(s,a) = \Gamma + \pi_{b,h}(a\mid s)(1-\Gamma)$. So, $\alpha_h(s,a) \geq \frac{1}{\Gamma}$ and $\beta_h(s,a) \leq \Gamma$ for all $s,a,h$. In particular, $\alpha_{min} \geq \frac{1}{\Gamma} = \frac{1}{1+\varepsilon}$ and $\beta_{max} \leq \Gamma = 1 + \varepsilon$.

In particular, we have the following bound.
$$\frac{1 + \varepsilon H - (1+\varepsilon)^H}{\varepsilon} \leq V_1^{\pi_e}(s) - \sum_a\pi_{e,1}(a \mid s)\hat{f}_1(s,a) \leq \frac{\frac{1}{(1+\varepsilon)^H} - (1-\varepsilon H)}{\varepsilon}$$

We know that we have the following bounds for small $\varepsilon$: $(1+\varepsilon)^H \geq 1+\varepsilon H + O(\varepsilon H^2)$ and $\frac{1}{(1+\varepsilon)^H} \leq 1 - \varepsilon H + O(\varepsilon H^2)$, giving us the following bound for small $\varepsilon$.

$$| V_1^{\pi_e}(s) - \sum_a\pi_{e,1}(a \mid s)\hat{f}_1(s,a) | \leq O(\varepsilon H^2)$$

\end{proof}

\begin{remark}\label{rem:more-conservative-FQE}
For any $\varepsilon$, the lower bound $\frac{1 + \varepsilon H - (1+\varepsilon)^H}{\varepsilon} \leq -\frac{\varepsilon H^2}{2}$, and thus we need to be at least as conservative as subtracting $\frac{\epsilon H^2}{2}$ from the FQE estimate to get a lower bound, if not more. This remark will be used in Section~\ref{sec:experiments}.
\end{remark}

We further remark in Section~\ref{sec:history-ind-confounders} that the bound in the theorem is data-oblivious, being only dependent on the confounding sensitivity model and horizon, and note that the other two methods below (CFQE and MB) both produce bounds at least as tight as this one. 

\subsection{Proof of CFQE Error Bounds, Theorem \ref{thm:err_CFQE}}

We recall the theorem below.
\errCFQE*

\begin{proof} In the limit of infinite data, the true value of $g_h$ always lies in the set $\tilde{B}_{sa,h}$ by the sensitivity assumption. So, CFQE trivially gives a lower bound on the true value function in the limit of infinite data. We now give bounds on its error below.

We define the error term at each step by $\text{err}_h(s,a):=\max_{s,a} Q_h^{\pi_e}(s,a) - \hat{f}_h(s,a)$, where here $f$ is generated by CFQE. 
We claim that 
\begin{align}
    \text{err}_h(s,a) = (H-h) - \frac{\alpha_{min}}{\beta_{max}} - \cdots - \frac{\alpha_{min}^{H-h}}{\beta_{max}^{H-h}} \label{CFQE_recursion}.
\end{align}
Then, the following bound follows for any $\varepsilon$.
\begin{align*}
    V_1^{\pi_e}(s) - \sum_a\pi_e(a\mid s)\hat{f}_1(s,a) &\leq H-1 - \frac{\alpha_{min}}{\beta_{max}} - \cdots - \frac{\alpha_{min}^{H-1}}{\beta_{max}^{H-1}}\\
    &\leq H- \sum_{i=0}^{H-1}\frac{1}{(1+\varepsilon)^{2i}}\\ 
    &\leq 2\varepsilon H^2
\end{align*}
This completes the proof since by induction, $\hat{f}_h(s,a) \leq Q_(s,a)$ for all $h$, and so we already have the lower bound $V_1^{\pi_e}(s) - \sum_a\pi_e(a\mid s)\hat{f}_1(s,a) \geq 0$. Thus, it remains to prove \ref{CFQE_recursion}.

At step $H$ of CFQE, we have
\begin{align*}
    \hat{f}_H(s,a) = r(s,a).
\end{align*}
Then as in the previous proof, the error at step $H$ is given by $\text{err}_H(s,a) = 0$.

At step $h+1$, suppose $\text{err}_{h+1}(s,a) = (H-h-1) - \frac{\alpha_{min}}{\beta_{max}} - \cdots - \frac{\alpha_{min}^{H-h-1}}{\beta_{max}^{H-h-1}}$. Then for step $h$, we have the following chain of inequalities for $\err_h(s,a) = Q_h^{\pi_e}(s,a) - \hat{f}_{h}(s,a)$.
\begin{align*}
    &\sum_{s'}\prob_h(s'\mid s,a) V_{h+1}^{\pi_e}(s') \\
    & \qquad -\sum_{u,s'}\prob^{\pi_b}(s', u \mid s,a)\pi_{b,h}(a\mid s)g_h(s,a,s')\sum_{a'}\pi_{e, h+1}(a'\mid s')\hat{f}_{h+1}(s',a')\\
    &=\sum_{s'}\prob_h(s'\mid s,a) V_{h+1}^{\pi_e}(s') \\
    & \qquad -\sum_{u,s'}P_h^{\pi_b}(u \mid s, a)\prob_h(s'\mid s,u,a)\pi_{b,h}(a\mid s)g_h(s,a,s')\sum_{a'}\pi_{e, h+1}(a'\mid s')\hat{f}_{h+1}(s',a')\\
    &=\sum_{s'}\prob_h(s'\mid s,a) V_{h+1}^{\pi_e}(s') \\
    & \qquad -\sum_{u,s'}P_h(u \mid s)\frac{\pi_{b, h}(a\mid s,u)}{\pi_{b,h}(a\mid s)}\prob_h(s'\mid s,u,a)\pi_{b,h}(a\mid s)g_h(s,a,s')\sum_{a'}\pi_{e, h+1}(a'\mid s')\hat{f}_{h+1}(s',a')\\
    \leq & \sum_{s'}\prob_h(s'\mid s,a) V_{h+1}^{\pi_e}(s')\\
    & \qquad -\frac{\alpha_h(s,a)}{\beta_h(s,a)}\sum_{u,s'}P_h(u \mid s)\prob_h(s'\mid s,u,a)\sum_{a'}\pi_{e, h+1}(a'\mid s')(\text{err}_{h+1}-Q_{h+1}^{\pi_e}(s,a))\\
    = &\left(1-\frac{\alpha_h(s,a)}{\beta_h(s,a)}\right)\sum_{s'}\prob_h(s'\mid s,a)V_{h+1}^{\pi_e}(s') + \frac{\alpha_h(s,a)}{\beta_h(s,a)}\text{err}_{h+1}\\
    \leq & \left(1-\frac{\alpha_h(s,a)}{\beta_h(s,a)}\right)(H-h) + \frac{\alpha_h(s,a)}{\beta_h(s,a)}\left(H-h-1 - \frac{\alpha_{min}}{\beta_{max}} - \cdots - \frac{\alpha_{min}^{H-h-1}}{\beta_{max}^{H-h-1}}\right)\\
    \leq & \left(1-\frac{\alpha_{min}}{\beta_{max}}\right)(H-h) + \frac{\alpha_{min}}{\beta_{max}}\left(H-h-1 - \frac{\alpha_{min}}{\beta_{max}} - \cdots - \frac{\alpha_{min}^{H-h-1}}{\beta_{max}^{H-h-1}}\right)\\
    = & H-h - \frac{\alpha_{min}}{\beta_{max}} - \cdots - \frac{\alpha_{min}^{H-h}}{\beta_{max}^{H-h}}.
\end{align*}
 The first expression comes from using equation~\ref{eqn:q-next-term} as well as explicitly computing the $\argmin$ involved in CFQE in the limit of infinite data, analogous to the proof of Theorem~\ref{thm:err_FQE} above. In the first inequality, we use the facts that $\frac{\pi_{b, h}(a\mid s,u)}{\pi_{b,h}(a\mid s)} \geq \frac{1}{\beta_h(s,a)}$ and $\pi_{b,h}(a\mid s)g_h(s,a,s') \leq \alpha_h(s,a)$. In the equality after that, we use the definition of $\err_h(s,a)$. In the second inequality, we use equation~\ref{eqn:q-next-term}. 

Thus, $\text{err}_h(s,a) = H-h - \frac{\alpha_{min}}{\beta_{max}} - \cdots - \frac{\alpha_{min}^{H-h}}{\beta_{max}^{H-h}}$ and \eqref{CFQE_recursion} is proved.
\end{proof}

\newpage 

\section{Model-Based Method} \label{sec:mbproofs}
\subsection{General Memoryless Version}

Notice that the model-based method leverages the fact that the \textit{marginalized transition dynamics} are stationary. In particular, we only need $\prob_h(s' \mid s,a,u)$ and $P_h(u \mid s)$ to be stationary, since this makes the marginalized transition structure stationary. In that light, we discuss here the version of the method where $\pi_b$ and $\pi_e$ are non-stationary.

Consider the observed transition structure at timestep $h$, given by $\hat{\prob}_h^{\pi_b}$, denote by $\hat{\pi}_{b,h}$ the observed behavior policy and by $\hat{\alpha}_h(s,a)$ and $\hat{\beta}_h(s,a)$ the versions of $\alpha_h(s,a)$ and $\beta_h(s,a)$ computed using $\hat{\pi}_{b,h}$. Let $\pi_e$ also be non-stationary. For the model to improve over CFQE, we still need $P_h(u \mid s)$ to be the same for all timesteps $h$, so that the marginalized transition structure is stationary.

Define $\mathcal{G}_h := \{\prob: \hat{\alpha}_h(s,a)\hat{\prob}_h^{\pi_b}(s' \mid s,a) \leq \prob(s' \mid s,a) \leq  \hat{\beta}_h(s,a)\hat{\prob}_h^{\pi_b}(s' \mid s,a), \ \forall s, a, s'\}$

Note that in the limit of infinite data, the true marginalized transition structure satisfies the following relation for each $h$.
$${\alpha}_h(s,a){\prob}_h^{\pi_b}(s' \mid s,a) \leq \prob(s' \mid s,a) \leq  {\beta}_h(s,a){\prob}_h^{\pi_b}(s' \mid s,a),\ \forall s,a,s'$$
So, in the limit of infinite data, the true marginalized structure lies in $\cap_h \mathcal{G}_h$. We then define this to be $\mathcal{G} := \cap_h \mathcal{G}_h$ even in the finite sample case.

With this as our $\mathcal{G}$, we have the same program for obtaining a model-based lower bound on the value function.
\begin{gather} \label{opt-hist-independent}
	\min_{V_1(s_0), V_2, \ldots, V_H, V_{H+1} = 0,\prob} V_1(s_0) \\ 
	\;\; \text{s.t. } \prob\in\mathcal{G}, \;\;
	\sum_{s'}\prob(s'\mid s,a) = 1 \;\; \forall s,a.\notag \\
	V_h(s) = \pi_{e,h}(\cdot\mid s)^T(R_s + \prob_s V_{h+1}(\cdot)) \;\; \forall h \in \{1,...,H\}, s \notag
\end{gather}

\begin{remark}
Note that assuming stationarity of $\pi_b$ allows us to use data across timesteps to estimate a universal $\hat{\prob}^{\pi_b}$, which helps with finite samples in practice.
\end{remark}

We present our proofs below for stationary $\pi_b$ and $\pi_e$ for clarity, noting that they can easily be modified for general memoryless $\pi_b$ and $\pi_e$ in a similar vein as the proofs for CFQE.

\subsection{Confidence Interval for State Transitions}
We can use the following lemma to modify our definition of the set $\mathcal{G}$ to use confidence intervals instead of point estimates. We show that both methods converge to the lower bound obtained with infinite data. However, using Hoeffding confidence intervals to modify $\mathcal{G}$ ensures that for any amount of data, the output of the model-based method is a true lower bound on the value function. In the version that uses point estimates of $\pi_b$ and $\prob^{\pi_b}$, we only get estimates of a lower bound with finite data. 

Let $N(s)$ and $N(s,a)$ be the counts of $s$ and $(s,a)$ in the data.

\begin{lemma} [Confidence Interval for State Transitions] \label{ineq:P_set} 
For $\Delta_\pi := \sqrt{\frac{1}{2N^*(s)}\log(\frac{2SA}{\delta_1})}$, 
$\Delta_\prob := \sqrt{\frac{1}{2N^*(s,a)}\log(\frac{2S^2A}{\delta_2})}$, 
bounds $\alpha_{\delta_1}(s,a) := 1/\Gamma - (1-1/\Gamma)(\hat{\pi}_{b}(a|s) + \Delta_\pi)$ and $\beta_{\delta_1}(s,a) := \Gamma + (1-\Gamma)(\hat{\pi}_{b}(a|s) + \Delta_\pi)$, 
and $N^*(s) = -\operatorname{logmeanexp}(\{-N(s_1), ...\})$, 
$\prob(s' | s,a)$ falls between $\alpha_{\delta_1}(s,a)(\hat{\prob}^{\pi_b}(s' | s,a)-\Delta_\prob)$ and $\beta_{\delta_1}(s,a)(\hat{\prob}^{\pi_b}(s' | s,a)+\Delta_\prob)$ 
with probability at least $1- \delta_1 - \delta_2$.
\end{lemma}

\begin{proof}
We attempt to use the data collected by $\pi_b$ to construct a confidence interval for $\hat{\prob}(s'\mid s,a)$ that also takes into account estimation error in the bounds $\alpha_{\delta_1}(s,a)$ and $\beta_{\delta_1}(s,a)$.
We consider below empirical estimation for $\prob(s'\mid s,a)$ and $\pi_b(a|s)$ using data collected by $\pi_b$:
\begin{align*}
	\hat{\prob}^{\pi_b}(s'\mid s,a) = \frac{N(s,a,s')}{N(s,a)}, \;\;\;
	\hat{\pi_b}(a|s) = \frac{N(s,a)}{N(s)}
\end{align*}
where $N(s,a,s'):=\sum_{i=1}^{n}\mathbb{1}_{\{s_i=s,a_i=a,s_i'=s'\}}$, $N(s,a):=\sum_{i=1}^{n}\mathbb{1}_{\{s_i=s,a_i=a\}}$, and $N(s):=\sum_{i=1}^{n}\mathbb{1}_{\{s_i=s\}}$.

Note that $\prob(s'\mid s,a) = \sum_u P(u \mid s)\prob(s'\mid s,u,a)$, while $\prob^{\pi_b}(s'\mid s,a) = \sum_u \prob^{\pi_b}(u\mid s,a)\prob(s'\mid s,u,a)$. In $\pi_b$, $u$ and $a$ are dependent.

We also have:
\begin{align*}
	\prob^{\pi_b}(s'\mid s,a) = \sum_u \prob^{\pi_b}(s',u\mid s,a) &= \sum_u \prob^{\pi_b}(u\mid s,a)\prob(s'\mid s,u,a)\\
	&= \sum_u P(u \mid s)\frac{\pi_b(a\mid s,u)}{\pi_b(a\mid s)}\prob(s'\mid s,u,a).
\end{align*}
By Assumption~\ref{assump:sensitivity}, 
\begin{align}
	\frac{1}{\beta(s,a)}\prob(s'\mid s,a) \leq \prob^{\pi_b}(s'\mid s,a) \leq \frac{1}{\alpha(s,a)}\prob(s'\mid s,a) \\
	\alpha(s,a)\prob^{\pi_b}(s'\mid s,a) \leq \prob(s'\mid s,a) \leq \beta(s,a)\prob^{\pi_b}(s'\mid s,a).  \label{eqn:Pbound}
\end{align}

We claim that by Hoeffding's inequality and the union bound, with probability at least $1-\delta_1-\delta_2$,
\begin{align*}
	\left|\hat{\pi_b}(a \mid s) - \pi_b(s \mid a)\right| &\leq \sqrt{\frac{1}{2N^*(s)}\log\left(\frac{2SA}{\delta_1}\right)} = \Delta_\pi\\
	\left|\hat{\prob}^{\pi_b}(s'\mid s,a) - \prob^{\pi_b}(s'\mid s,a)\right| &\leq \sqrt{\frac{1}{2N^*(s,a)}\log\left(\frac{2S^2A}{\delta_2}\right)} = \Delta_\prob \numberthis \label{eqn:Hoeffding-bounds}
\end{align*}

where $N^*(s) = -\operatorname{logmeanexp}(\{-N(s_1), ...\})$ and 
$N^*(s,a) = -\operatorname{logmeanexp}(\{-N(s_1,a_1), ...\})$. 

We illustrate this by showing the result for $\hat{\prob}^{\pi_b}(s'\mid s,a)$, and the other case follows analogously.

\begin{align*}
    \prob(\exists s', s, a \;\; s.t. \; \;|\hat{\prob}^{\pi_b}(s'\mid s,a) - \prob^{\pi_b}(s'\mid s,a)| \leq \epsilon)  
    &\leq \sum_{s',s,a} \prob(|\hat{\prob}^{\pi_b}(s'\mid s,a) - \prob^{\pi_b}(s'\mid s,a)| \leq \epsilon) \\
    &\leq \sum_{s',s,a} 2 \exp\{ -2\epsilon^2 N(s,a) \} \\
    &= S \sum_{s,a} 2 \exp\{ -2\epsilon^2 N(s,a) \} \\
    &\leq 2 S^2 A \exp\{ -2\epsilon^2 N^*(s,a) \} 
    = \delta
\end{align*}

for some $N^*$ that satisfies the last inequality above. Various choices for $N^*$ exist. Perhaps the most obvious choice is the $\min$ function, though it can be shown that $-\operatorname{logmeanexp}(-x)$ is optimal, as:

\begin{align*}
    x^* \;\; s.t. \;\; \sum_n e^{X_n} = ne^{x^*} \iff e^{x^*} = \frac{1}{n}\sum_n e^{X_n} \iff \operatorname{logmeanexp}(X_1, ..., X_n) = x^*
\end{align*}

The $\operatorname{logmeanexp}$ function returns a value between the maximum and the mean, and in our case, we use it to obtain a soft approximation to the minimum that provides a less conservative bound than using the minimum of counts over all states (or states and actions). 

Combining our inequalities~\ref{eqn:Hoeffding-bounds} and ~\ref{eqn:Pbound} with the definitions of $\alpha(s,a)$ and $\beta(s,a)$, we have our result.
\end{proof}

\subsection{Solving~\eqref{opt} Gives Better Lower Bound than Confounded FQE, Proof of Theorem \ref{thm:model_base}}

Recall Theorem~\ref{thm:model_base} below.

\errModelBased*

We consider the infinite sample setting, which means:
\begin{align*}
	\mathcal{G} = \{\prob: \alpha(s,a) \leq \frac{\prob(s'\mid s,a)}{\prob^{\pi_b}(s'\mid s,a)}\leq \beta(s,a), \text{ for }\forall s,a,s'\}
\end{align*}

The key to the proof is the observation that we can always get a valid $g_h$ from a valid $\prob \in \mathcal{G}$ by setting $g_h(s,a,s') := \frac{\prob(s'\mid s,a)}{\prob^{\pi_b}(s'\mid s,a)\pi_b(a\mid s)}$, which formalizes the intuition that the uncertainty set $\mathcal{G}$ for $\prob$ is tighter. Since we are in the stationary case, we drop all unnecessary $h$ in subscripts.

\begin{proof}
We denote the solution of \eqref{opt} in the infinite-sample setting by $\tilde{V}_1,\ldots,\tilde{V}_H,\tilde{V}_{H+1}, \tilde{\prob}$. We will show that $\tilde{V}_1$ gives a lower bound on the true value function that is larger than the lower bound given by CFQE. That is, if the iterates of CFQE are $\hat{f}_h(s,a)$, then $\sum_{a} \pi_e(a \mid s)\hat{f}_1(s,a) \leq \tilde{V}_1(s) \leq V_1^{\pi_e}(s)$. Combining this with Theorem~\ref{thm:err_CFQE} gives us the whole theorem.

First note that in the infinite data setting, the marginalized transition kernel lies in $\mathcal{G}$, so the optimization problem minimizes $V_1$ over values of $\prob$ that include the true marginalized transition structure. Thus, we trivially get that $\tilde{V}_1(s) \leq V_1^{\pi_e}(s)$.

We now prove that $V_h(s)\geq \sum_a\pi_e(a\mid s)\hat{f}_h(s,a)$ holds for all $h$ by induction. Note that the argument below also works for the finite-sample case by merely replacing every quantity associated with $\pi_b$ (such as $\prob^{\pi_b}$) by its finite sample version.

For $h = H+1$:
\begin{align*}
	\tilde{V}_{H+1}(s) = 0 \geq 0 = \sum_a \pi_e(a \mid s) \hat{f}_{H+1}(s,a)
\end{align*}

Suppose we have $\tilde{V}_{h+1}(s) \geq \sum_a\pi_e(a\mid s)\hat{f}_{h+1}(s,a)$.
Then for step $h$:
\begin{align*}
	\tilde{V}_{h}(s) = \sum_a\pi_e(a\mid s) \left[R(s,a) + \sum_{s'}\tilde{\prob}(s'\mid s,a)\tilde{V}_{h+1}(s')\right]
\end{align*}

\begin{align*}
	&\hat{f}_h(s,a)
    = \min_{g\in\tilde{\B}_{sa}} \left(\sum_{u,s'}  \prob^{\pi_b}(s', u \mid s,a) CFQE(\hat{f}_{h+1}, g)\right)\\
	&\leq \sum_{u,s'} \prob^{\pi_b}(s', u \mid s,a)\frac{\tilde{\prob}(s'\mid s,a)}{\prob^{\pi_b}(s'\mid s,a)}\left[R(s,a) + \sum_{a'}\pi_e(a'\mid s')\hat{f}_{h+1}(s',a')\right]\\
	& = \sum_{s'}\tilde{\prob}(s'\mid s,a)\left[R(s,a) + \sum_{a'}\pi_e(a'\mid s')\hat{f}_{h+1}(s',a')\right] \leq R(s,a) + \sum_{s'} \tilde{\prob}(s'\mid s,a) \tilde{V}_{h+1}(s').
\end{align*}
where $$CFQE(\hat{f}_{h+1}, g) := \left(
    \sum_{s'}\pi_b(a\mid s)g(s,a,s')\left[R(s,a) + \sum_{a'}\pi_e(a'\mid s')\hat{f}_{h+1}(s',a')\right]\right)$$
The first inequality in above is achieved by setting $g(s,a,s') = \frac{\tilde{\prob}(s'\mid s,a)}{\prob^{\pi_b}(s'\mid s,a)\pi_b(a\mid s)}$.
It's easy to check that by this choice, $g(s,a,\cdot)\in\tilde{\B}_{sa}$ by \eqref{eqn:Pbound}.
The second inequality is by the induction hypothesis.
Thus, we have $\tilde{V}_h(s) \geq \sum_a \pi_e(a\mid s)\hat{f}_h(s,a)$.

By induction, $\tilde{V}_1(s) \geq \sum_a \pi_e(a\mid s)\hat{f}_1(s,a)$, which means the lower bound provided by \eqref{opt} is always no worse than confounded FQE (Alg. \ref{algo:cFQE}).

\end{proof}

\subsection{Worst-Case Error for the Model-Based Method, An Independent Alternative Proof}
In this section, we give an alternative proof of the fact that the output of \eqref{opt} satisfies $|V_1^{\pi_e}(s) - \tilde{V}_1| = O(\varepsilon H^2)$ for $\Gamma = 1 + \varepsilon$ without comparing to CFQE. Again, recall that we consider the infinite sample setting, which means the following.
\begin{align*}
\mathcal{G} = \{\prob: \alpha(s,a) \leq \frac{\prob(s'\mid s,a)}{\prob^{\pi_b}(s'\mid s,a)}\leq \beta(s,a), \text{ for }\forall s,a,s'\}
\end{align*}

\begin{proof}
	By definition, we know $V_H^{\pi_e}(s) = \tilde{V}_H(s)$ for all $s$. We define $\delta_h = \max_s |V_h^{\pi_e}(s) - \tilde{V}_h(s)|$. Note that $\delta_H =0$.
	Next, consider $|V_h^{\pi_e}(s) - \tilde{V}_h(s)|$:
	\begin{align*}
		\delta_h &:=\max_s |V_h^{\pi_e}(s) - \tilde{V}_h(s)| \\
        &= \max_s |\sum_a\pi_e(a\mid s)\sum_{s'}\prob(s'\mid s,a)V_{h+1}(s') - \sum_a\pi_e(a\mid s)\sum_{s'}\tilde{\prob}(s'\mid s,a)\tilde{V}_{h+1}(s')|\\
		& \leq \max_s|\sum_a\pi_e(a\mid s)\sum_{s'}\prob(s'\mid s,a)V_{h+1}(s') - \sum_a\pi_e(a\mid s)\sum_{s'}\tilde{\prob}(s'\mid s,a)V_{h+1}(s')|\\
		& \qquad + \max_s|\sum_a\pi_e(a\mid s)\sum_{s'}\tilde{\prob}(s'\mid s,a)V_{h+1}(s') - \sum_a\pi_e(a\mid s)\sum_{s'}\tilde{\prob}(s'\mid s,a)\tilde{V}_{h+1}(s')|\\
		&= \max_s|\sum_a\pi_e(a\mid s)\sum_{s'}\prob(s'\mid s,a)V_{h+1}(s') - \sum_a\pi_e(a\mid s)\sum_{s'}\tilde{\prob}(s'\mid s,a)V_{h+1}(s')|\\
		& \qquad + \delta_{h+1}\\
		&\leq (\beta_{\text{max}} - \alpha_{\text{min}})(H-h) + \delta_{h+1},
	\end{align*}
	where $$\beta_{\text{max}} := \max_{s,a} \Gamma + \pi_b(a\mid s)(1-\Gamma) \leq 1+\varepsilon$$ and $$\alpha_{\text{min}}:=\min_{\pi_b(a\mid s)}\frac{\varepsilon}{1+\varepsilon}\pi_b(a\mid s)+\frac{1}{1+\varepsilon} \geq \frac{1}{1+\varepsilon}$$
	It is easy to check $\beta_{\text{max}}-\alpha_{\text{min}} = \varepsilon + \frac{\varepsilon}{1+\varepsilon} = O(\varepsilon)$ (ignoring higher order terms of $\varepsilon$).
	So, we get that $\delta_h\leq  O(\varepsilon(H-h)) + \delta_{h+1}$ from $h = 1,\ldots,H$. So, we have that
	$$\delta_1\leq O(\varepsilon H^2)$$
	
\end{proof}
	
\subsection{Consistency of the Model-Based Method}

We first prove this extremely elementary and useful geometric lemma.

\begin{lemma}\label{lem:hausdorff-minima}
If a function $f: X \to \mathbb{R}$ on a Hausdorff metric space X is continuous (resp. Lipschitz), then $f_{\min}: Comp(X) \to \mathbb{R}$ given by $f_{\min}(K) := \inf_{x \in K} f(x)$ is also continuous (resp. Lipschitz) in the Hausdorff metric on the space $Comp(X)$ of compact subsets of $X$. The same holds for $f_{\max}(K) := \sup_{x \in K} f(x)$.
\end{lemma}
\begin{proof}
We prove this for $\alpha$-Lipschitz $f$ and $f_{\min}$, the other cases are similar. Consider compact sets $K_1$ and $K_2$, so that the infima are attained at $x_i \in K_i$. This means that $f_{\min}(K_i) = f(x_i)$. Since $K_j$ are closed, we have points $u_j$ that attain the closest distance from $x_i$ to $K_j$. Combining these, we know that  $$d(x_i, K_j) \leq d_{Haus}(K_i, K_j)$$ and $$d(x_i, u_j) = d(x_i, K_j) := inf_{u \in K_j} d(x_i, u)$$\\
Using the Lipschitzness of $f$, $$|f(x_i) - f(u_j)| \leq \alpha d(x_i, u_j) \leq \alpha d_{Haus}(K_i, K_j)$$
Also, $f(u_j) \geq f(x_j)$ by definition of $x_j$, since $u_j \in K_j$ and $x_j$ minimizes $f$ over $K_j$. So, 
$$f(x_i) \geq f(u_j) - \alpha d_{Haus}(K_i, K_j) \geq f(x_j) - \alpha d_{Haus}(K_i, K_j)$$ This holds for $(i,j) = (1,2), (2,1)$, so we get that $$|f_{\min}(K_i) - f_{\min}(K_j)| = |f(x_i) - f(x_j)| \leq \alpha d_{Haus}(K_i, K_j)$$
\end{proof} 
We use Lemma \ref{lem:hausdorff-minima} along with the fact that the objective function is Lipschitz. We will prove it for the version of the Model-Based method incorporating Hoeffding-based bounds (which are incorporated to give finite sample guarantees). The proof for the version with point estimates of the relevant quantities is in fact easier and subsumed by this by setting $\Delta_\pi=\Delta_\prob=0$. We first need the lemma below, which will we later combine with Lemma \ref{lem:hausdorff-minima}.

\begin{lemma} \label{lem:feasible-set-hausdorff-bound}
Let the feasible region given by the values of $P_{\pi_b}, \alpha(s, a)$ and $\beta(s,a)$ in the limit of infinite data be $F$. Let the feasible region obtained using our finite sample estimates in Lemma \ref{ineq:P_set} be $\hat{F}$. Then there is a constant $K$ depending on $\Gamma$ so that
$$d_{Haus}(F, \hat{F}) \leq 2S^2A\Gamma\left(|\prob^{\pi_b}(s' \mid s, a) - \hat{\prob}^{\pi_b}(s' \mid s, a)| + |\pi_b(s \mid a) - \hat{\pi_b}(s \mid a)| + \Delta_\prob + \Delta_\pi\right)$$
\end{lemma}

Notice that this also applies to the case of replacing the Hoeffding-based intervals by the point estimates, since that merely involves replacing $\Delta_\pi$ and/or $\Delta_\prob$ by $0$.

\begin{proof}

 Notice that the condition $$\sum_{s'} \prob(s' \mid s, a) = 1$$ is identical across both sets, so the difference is only induced by the infinite-sample $\mathcal{G}_\infty$ and the finite sample $\mathcal{G}$. That is, for $\prob \in \mathcal{G}_\infty$, we have the following
$$\alpha(s,a)\prob^{\pi_b}(s' | s,a) \leq \prob(s' \mid s,a) \leq \beta(s,a)\prob^{\pi_b}(s' | s,a)$$ 
Let's call the interval above $I_{s,a}$. For $\prob \in \mathcal{G}$, we instead have
$$\alpha_{\delta_1}(s,a)(\hat{\prob}^{\pi_b}(s' | s,a)-\Delta_\prob) \leq \prob(s' \mid s, a) \leq \beta_{\delta_1}(s,a)(\hat{\prob}^{\pi_b}(s' | s,a)+\Delta_\prob)$$
We can check that using the inequalities above, the following hold for $\hat{w} \in \hat{F}$:

\begin{itemize}
    \item If $\hat{w}_{s', s,a} < \alpha(s,a)\prob^{\pi_b}(s' \mid s, a)$, then 
\begin{align*}
    &d(\hat{w}_{s', s,a}, I_{s,a}) \\
    &\leq \alpha(s, a)|\prob^{\pi_b}(s' \mid s, a) - \hat{\prob}^{\pi_b}(s' \mid s, a)| + \alpha(s,a)\Delta_\prob \\
    &\;\;\;\;\;\;\;\;\;\;
    +(\hat{\prob}^{\pi_b}(s' \mid s, a)+\Delta_\prob)|\alpha(s,a) - \alpha_{\delta_1}(S, a)|\\
    & \leq |\prob^{\pi_b}(s' \mid s, a) - \hat{\prob}^{\pi_b}(s' \mid s, a)| + \Delta_\prob + 2\left(1-\frac{1}{\Gamma}\right)(|\pi_b(s \mid a) - \hat{\pi_b}(s \mid a)| + \Delta_\pi)\\
    & \leq |\prob^{\pi_b}(s' \mid s, a) - \hat{\prob}^{\pi_b}(s' \mid s, a)| + K_1|\pi_b(s \mid a) - \hat{\pi_b}(s \mid a)| + \Delta_\prob +K_1\Delta_\pi
\end{align*}
where $K_1 = 2\left(1-\frac{1}{\Gamma}\right)$.
    \item If $\hat{w}_{s', s,a} > \beta(s,a)\prob^{\pi_b}(s' \mid s, a)$ then we get terms using $\beta$, so that we have 
$$d(\hat{w}_{s', s,a}, I_{s,a}) \leq K_3 |\prob^{\pi_b}(s' \mid s, a) - \hat{\prob}^{\pi_b}(s' \mid s, a)| + K_2|\pi_b(s \mid a) - \hat{\pi_b}(s \mid a)| + K_3\Delta_\prob +K_2\Delta_\pi$$ 
with $K_2 = 2(\Gamma - 1)$ and $K_3 = \Gamma$
\item In the third case, $\hat{w}_{s', s,a} \in I_{s,a}$, so $d(\hat{w}_{s', s,a}, I_{s,a}) = 0$
\end{itemize}
Combining these and noting that $2\Gamma \geq K_1, K_2, K_3$, we have that
$$d(\hat{w}_{s', s,a}, I_{s,a}) \leq 2\Gamma(|\prob^{\pi_b}(s' \mid s, a) - \hat{\prob}^{\pi_b}(s' \mid s, a)| + |\pi_b(s \mid a) - \hat{\pi_b}(s \mid a)| + \Delta_\prob +\Delta_\pi)$$
This means that by the triangle inequality, for any matrix/vector norm on $\mathbb{R}^{S^2A}$,
\begin{align*}
    d(\hat{w}, F) & = d(\hat{w}, \prod_{s',s,a} I_{s,a}) \leq S^2A \max_{s', s,a} d(\hat{w}_{s', s,a}, I_{s,a})\\
    & \leq 2S^2A\Gamma\left(|\prob^{\pi_b}(s' \mid s, a) - \hat{\prob}^{\pi_b}(s' \mid s, a)| + |\pi_b(s \mid a) - \hat{\pi_b}(s \mid a)| + \Delta_\prob +\Delta_\pi\right)
\end{align*}

Since $\hat{w} \in \hat{F}$ is arbitrary, 
$$d_{Haus}(F, \hat{F}) \leq 2S^2A\Gamma\left(|\prob^{\pi_b}(s' \mid s, a) - \hat{\prob}^{\pi_b}(s' \mid s, a)| + |\pi_b(s \mid a) - \hat{\pi_b}(s \mid a)| + \Delta_\prob +\Delta_\pi\right)$$
\end{proof}

We finally recall and prove our consistency result below.

\ModelBasedConsistent*

\begin{proof} To remind the reader of the precise sense in which "limit of infinite data" is used here, we mean that the behavior policy is exploratory, so that every $s,a$ has a non-zero probability of occurring in the trajectory. In particular $N(s), N(s,a) \to \infty$ as we observe infinitely many trajectories.

We know that our objective function is a polynomial in the entries of $w = \prob(\cdot \mid \cdot, \cdot)$. Since the entries of $w$ lie in $[0,1]$, the domain of our multivariate polynomial is compact and it is thus Lipschitz, since it is $C^1$. Let its Lipschitz constant be $\alpha$. Call the minimum in the infinite data case $\tilde{V}_1$ and the one in the finite sample case $\hat{V_1}$. Combining Lemma~\ref{lem:feasible-set-hausdorff-bound} with Lemma \ref{lem:hausdorff-minima}, we get that

\begin{align*}
    |\tilde{V}_1 - \hat{V_1}| 
    &\leq \alpha d_{Haus}(F, \hat{F}) \\ 
    &\leq 2\alpha S^2A \Gamma \left(|\prob^{\pi_b}(s' \mid s, a) - \hat{\prob}^{\pi_b}(s' \mid s, a)| + |\pi_b(s \mid a) - \hat{\pi_b}(s \mid a)| + \Delta_\prob +\Delta_\pi\right)
\end{align*} 

Note that as $N(s), N(s,a) \to \infty$, $|\prob^{\pi_b}(s' \mid s, a) - \hat{\prob}^{\pi_b}(s' \mid s, a)|, |\pi_b(s \mid a) - \hat{\pi_b}(s \mid a)| \to 0$ almost surely, and $\Delta_\prob, \Delta_\pi \to 0$. This implies that as $N(s), N(s,a) \to \infty$, $|\tilde{V}_1 - \hat{V_1}|  \to 0$ almost surely.

\end{proof}

\newpage

\section{Variations of The Model-Based Method} \label{sec:mb-varations}

\subsection{Relaxation of \eqref{opt}}

Recall that in \eqref{opt}, we solved a non-convex optimization problem with $H \cdot |\Sset| + 1$ Bellman backup constraints. If one were to not require the $\prob(s'|s, a)$ to stay constant at every step, one could sequentially solve $H \cdot |\Sset| + 1$ convex programs to obtain a lower bound that is looser than one obtained by \eqref{opt}. $\mathcal{G}_h$ is as defined in Appendix~\ref{sec:mbproofs}. Computationally, to compute policy values for each starting state, confounded FQE (Alg.~\ref{algo:cFQE}) solves $(H+1) \cdot |\Sset| \cdot |\A|$ linear programs, while Alg.~\ref{algo:MBRelax} below solves $(H+1) \cdot |\Sset|$ convex programs.

\begin{algorithm}[H]
	\centering
	\caption{Relaxation of Model-Based Method}
	\begin{algorithmic}[1]
		\STATE \textbf{input: } evaluation policy $\pi_e$, starting state $s_0$.
		\STATE \textbf{initialize:} $V_{H+1}\leftarrow 0$.
		\FOR{$h = H, H-1,\ldots,1$}
		\STATE 
		\begin{align*}
		V_h(s) &:= \min_{\prob_h \in \mathcal{G}_h} \sum_{a}\pi_{e,h}(a\mid s)\left[R(s,a) + \sum_{s'}\prob_h(s'\mid s,a)V_{h+1}(s')\right] \\
		&= \min_{\prob_h \in \mathcal{G}_h} \pi_{e,h}(\cdot\mid s)^T(R_s + \prob_{s,h} V_{h+1}(\cdot)).
		\end{align*}
		\ENDFOR
		\STATE \textbf{return} $V_1(s_0)$
	\end{algorithmic}
\label{algo:MBRelax}
\end{algorithm}

Notice that this is similar to confounded FQE (Alg.~\ref{algo:cFQE}) in that it optimizes over $\prob_h(s'|s, a)$ at \textit{each step}, instead of requiring it to stay constant for all $h=1,...H$. Consider the bijection $g_h(s,a,s') \leftrightarrow \frac{\prob_h(s' \mid s,a)}{\hat{\prob}^{\pi_b}_h(s' \mid s,a)\hat{\pi}_{b, h}(a \mid s)}$ between the uncertainty sets $\prod_h \tilde{B}_{sa, h}$ and $\prod_h \mathcal{G}_h$ for $g_1, \dots g_H$ and $\prob_1, \dots, \prob_H$ respectively. It is easy to check using the definitions of the sets that this is truly a bijection. We can see using this bijection and with an argument similar to the proof of Theorem~\ref{thm:model_base}, that the value estimates from this relaxation and CFQE are equal at each step. By the remark made in the proof of Theorem~\ref{thm:model_base}, this also holds for the finite sample versions.

\subsection{Projected Gradient Descent}

In a similar vein to Algorithm 4.1 in \cite{kallus2020confounding}, we provide a method to efficiently compute the lower bound with projected gradient descent.

Given an estimate of $\prob$, the corresponding estimate of $V_1(s_0)$ can be obtained by iteratively performing $H+1$ Bellman backups, each of which is dependent on $\prob$ itself. Each Bellman backup is obtained by translations and matrix multiplications of $\prob$. As such, $V_1(s_0)$ is differentiable with respect to $\prob$, and the gradient $\nabla_\prob V_1(s_0)$ can be easily obtained with modern autograd tools.

\begin{algorithm}[H]
	\centering
	\caption{Projected Gradient Descent for Model-Based Lower Bound}
	\begin{algorithmic}[1]
		\STATE \textbf{input: } evaluation policy $\pi_e$, empirical estimate of $\prob$, decaying learning rate $\eta_t$, starting state $s_0$.
		\STATE \textbf{initialize:} $V_{H+1}\leftarrow 0$.
	    \FOR{$t = 1, ..., N$}
		\FOR{$h = H, H-1,\ldots,1$}
		\STATE 
		\begin{align*}
	    V_h(s) &:= \sum_{a}\pi_e(a\mid s)\left[R(s,a) + \sum_{s'}\prob(s'\mid s,a)V_{h+1}(s')\right] \\
		&= \pi_e(\cdot\mid s)^T(R_s + \prob_s V_{h+1}(\cdot)).
		\end{align*}
		\ENDFOR
		\STATE $\prob \leftarrow \operatorname{Proj}_\mathcal{G}(\prob - \eta_t \nabla_{\prob} V_1(s_0))$
		\ENDFOR
		\STATE \textbf{return} the lowest $V_1(s_0)$ encountered.
	\end{algorithmic}
\label{algo:projOPE}
\end{algorithm}

\newpage

\section{FQE Does Not Work for Confounders with Memory}\label{sec:hist-dep-lower-bound}
We recall Theorem~\ref{thm:hist-dependent-lower-bound} below.

\HistDependentLowerBound*

\begin{proof}
We demonstrate that there exists a confounded MDP with non-memoryless confounders and a behavior policy $\pi_e$ where even under the limit of infinite data, if the estimate obtained using FQE is $\hat{f}_1(s,a)$ and the true value function is $V_1^{\pi_e}(s)$, then $V_1^{\pi_e}(s) - \sum_{a} \pi_e(a \mid s) \hat{f}_1(s,a) = O(H)$.

\paragraph{Environment:} 
\begin{itemize}
    \item Consider $S = \{s_1,s_2\}$, $A = \{a_1,a_2\}$, $U = \{u_0, u_{a_1}\}$, horizon $H$.
    \item Rewards: $r(s=s_1, a_1) = 1$, otherwise $0$ reward.
    \item Starting state: Let the starting state be $s_1$.
\end{itemize}

\textbf{Confounder distribution:} The confounder's distribution starts at $u_{a_1}$ and is induced by confounder transitions with memory. Specifically, consider the following confounder transitions.
\begin{itemize}
    \item If $u=u_{a_1}$ and the current action is $a_1$, stay in $u_{a_1}$.
    \item In all other cases, transition to $u_0$.
\end{itemize}

\textbf{State transitions:} $\prob(s_1 \mid s, a_1, u_{a_1}) = 1$ for any $s$, and for all other $s,a,u$, we have that $\prob(s_1 \mid s, a, u) = 1/H$ and $\prob(s_2 \mid s,a,u) = 1-1/H$

\textbf{Behavior policy:} Let $\pi_b(a \mid s, u) = \frac{1}{2}$ for any $s,a,u$.

\textbf{Evaluation policy:} Let $\pi_e(a_1 \mid s) = 1$.

\textbf{Policy values:} Notice that in the evaluation policy, we are always in $u_{a_1}$ and always take action $a_1$, so we are always in state $s_1$. Thus the reward at each step is $1$ and $V_1^{\pi_e}(s_1) = H$.

\textbf{FQE Output:} First note that to iterate through FQE for $\pi_e$, we need only compute $\hat{f}_h(s, a_1)$ for all $s,h$. Notice that under the behaviour policy, at timestep $h$, $\prob_{\pi_b, h}(u_{a_1}) = \frac{1}{2^{h-1}}$ and $\prob_{\pi_b, h}(u_0) = 1 - \frac{1}{2^{h-1}}$. We start with $\hat{f}_{H+1}(s,a) := 0$ and the update rule is given by

\begin{align*}
    \hat{f}_h(s,a) &= \E_{(s,a,s') \in \mathcal{D}_{\pi_b, h}} [r(s,a) + \sum_{a'} \pi_e(a' \mid s') \hat{f}_{h+1}(s',a')]\\
    &= \E_{(s,a,s') \in \mathcal{D}_{\pi_b, h}} [r(s,a) + \hat{f}_{h+1}(s',a_1)]\\
    &= r(s,a) + \sum_{s', u} \prob_{\pi_b,h}(s', u \mid s,a)\hat{f}_{h+1}(s',a_1)\\
    &= r(s,a) + \sum_{s', u}\prob(s' \mid s,a,u)\prob_{\pi_b, h}(u \mid s,a) \hat{f}_{h+1}(s',a_1)
\end{align*}

Note that for $u=u_0, u_{a_1}$ 
$$\prob_{\pi_b, h}(u \mid s,a) = \frac{\prob_{\pi_b, h}(s,a \mid u)\prob_{\pi_b, h}(u)}{\prob_{\pi_b, h}(s,a \mid u_0)\prob_{\pi_b, h}(u_0)+\prob_{\pi_b, h}(s,a \mid u_{a_1})\prob_{\pi_b, h}(u_{a_1})}$$

For $s=s_2$, $\prob(s_2, a \mid u_{a_1}) = 0$, so $\prob(u_{a_1} \mid s_2, a) = 0$. On the other hand, for $s_1, a_1$, we have the following.
$$\prob_{\pi_b, h}(u_{a_1} \mid s_1,a_1) = \frac{\frac{1}{2^{h-1}}}{\frac{1}{2H}\left(1-\frac{1}{2^{h-1}}\right) + \frac{1}{2^{h-1}}} \leq \min\left(1,\frac{2H}{2^{h-1}}\right)$$
Thus, $\prob_{\pi_b, h}(u_0 \mid s_1,a_1) \geq 1 - \frac{2H}{2^{h-1}}$

Thus, for $s_1, a_1$, the update rule is given by 
\begin{align*}
    \hat{f}_h(s_1,a_1) &= 1 + \frac{1}{H}\prob_{\pi_b, h}(u_0 \mid s_1, a_1)\hat{f}_{h+1}(s_1, a_1) + \left(1 - \frac{1}{H} \right)\prob_{\pi_b, h}(u_0 \mid s_1, a_1)\hat{f}_{h+1}(s_2, a_1) \\
    & \qquad + \prob_{\pi_b, h}(u_{a_1} \mid s_1, a_1)\hat{f}_{h+1}(s_1, a_1)\\
    &\leq 1 + \left(\frac{1}{H} + \min\left(1,\frac{2H}{2^{h-1}}\right)\right)\hat{f}_{h+1}(s_1, a_1) + \left(1 - \frac{1}{H}\right)\hat{f}_{h+1}(s_2, a_1)
\end{align*}
For $s_2$, it is given by
\begin{align*}
    \hat{f}_h(s_2,a_1) &= 1 + \frac{1}{H}\prob_{\pi_b, h}(u_0 \mid s_1, a_1)\hat{f}_{h+1}(s_1, a_1) + \left(1 - \frac{1}{H} \right)\prob_{\pi_b, h}(u_0 \mid s_1, a_1)\hat{f}_{h+1}(s_2, a_1) \\
    & \qquad + \prob_{\pi_b, h}(u_{a_1} \mid s_1, a_1)\hat{f}_{h+1}(s_1, a_1)\\
    &= \frac{1}{H}\hat{f}_{h+1}(s_1, a_1) + \left(1-\frac{1}{H}\right)\hat{f}_{h+1}(s_2, a_1)
\end{align*}

We can use these to perform a straightforward but tedious calculation and inductively verify that for $h \geq 2\log(H) + 6$, $\hat{f}_h(s_1,a_1) \leq 1+ \frac{2H-2h}{H}$ and $\hat{f}_h(s_2, a_1) \leq \frac{2H-2h}{H}$. Induction starts at $h=H$ and works backwards. For $h \leq 2\log(H) + 6$, we use the simple upper bounds on the FQE recursion.
$$\hat{f}_h(s_1,a_1) \leq 1 + \max(\hat{f}_{h+1}(s_1, a_1),\hat{f}_{h+1}(s_2, a_1))$$
$$\hat{f}_h(s_2,a_1) \leq \max(\hat{f}_{h+1}(s_1, a_1), \hat{f}_{h+1}(s_2, a_1))$$
In particular,
$$\max(\hat{f}_{h}(s_1, a_1), \hat{f}_{h}(s_2, a_1)) \leq 1 + \max(\hat{f}_{h+1}(s_1, a_1), \hat{f}_{h+1}(s_2, a_1))$$
This gives us the following relation.
$$\hat{f}_1(s_1, a_1) \leq \max(\hat{f}_1(s_1, a_1), \hat{f}_1(s_2, a_1)) \leq (2\log H + 6) + 1 + \frac{2(H - (2\log H +6))}{H} \leq 2\log H + 9$$

In particular, FQE gives an underestimate of the value and its estimation error is $$V^{\pi_e}_1(s_1) - \sum_a \pi_e(a \mid s) \hat{f}_1(s_1,a) = O(H)$$

\end{proof}
\newpage

\section{Proof of Consistency for Clustering OPE, Theorem \ref{thm:cluster-ope-sample-complexity}} \label{sec:clustering-ope}
We first rephrase the end-to-end clustering guarantee from \cite{ambuj2022mixmdp} in our context.

\begin{thm*}
Under Assumptions~\ref{assump:global_u}, \ref{assump:mixing}, and \ref{assump:model_sep}, there are constants $H_0$, $N_0$ depending polynomially on $\frac{1}{\alpha}, \Delta, \frac{1}{\min_u P(u)}, \log(1/\delta)$, so that for $n \geq U^2SN_0\log(1/\delta)$ trajectories of length $H \geq H_0t_{mix}\log(n)$, we recover all clusters of trajectories exactly with probability at least $1-\delta$.
\end{thm*}

We now recall Theorem~\ref{thm:cluster-ope-sample-complexity}. 


\ClusterOPESampleComplexity*

As discussed in Section~\ref{sec:history-dependent-confounders}, we prove a more general version of this, in the form of the theorem below. Assume that we instantiate Algorithm~\ref{algo:clusterOPE} with an OPE estimator that requires an assumption $A(b)$ parameterized by a vector $b$ and has sample complexity $N_2(\delta, \epsilon, b)$.

\begin{thm*}
Under Assumptions~\ref{assump:global_u}, \ref{assump:mixing}, \ref{assump:model_sep}, and $A(b)$, there are constants $H_0$, $N_0$ depending polynomially on $\frac{1}{\alpha}, \Delta, \frac{1}{\min_u P(u)}, \log(1/\delta)$, so that for $n$ trajectories of length $H \geq H_0t_{mix}\log(n)$, we have that $|\hat{V}_1(s_0 ; \pi_e) - V_1(s_0 ; \pi_e)| < \epsilon$ with probability at least $1-\delta$ if

$$n \geq \Omega\left(\max\left(U^2SN_0\log(1/\delta), \frac{\log(U/\delta)}{\min(\epsilon^2/H^2, \min_u P(u)^2)}, N_2(\delta/U, \epsilon, b) \right)\right).$$
\end{thm*}

\begin{proof}
Note that $V_1(s_0 ; \pi_e) = \mathbb{E}_{u}[V_1(s_0 ; u, \pi_e)] = \sum_u P(u) V_1(s_0 ; u, \pi_e)$. Using the clustering guarantee from \cite{ambuj2022mixmdp} (rephrased above), we know that for the same $H_0$ and $N_0$ as in the clustering guarantee, given $n \geq N(\delta) = U^2SN_0\log(1/\delta)$ trajectories of length $H \geq H_0t_{mix}\log(n)$, we recover clusters $C_1,...,C_U$ consisting of trajectories with the same confounders with probability at least $1-\delta$. Recall that $H_0$ is not explicitly dependent on $S, A$ and $t_{mix}$, but could depend on the model.

We only identify the confounder labels in each trajectory up to permutation upon obtaining exact clustering, but for any permutation $\sigma \in S_U$, $\sum_{u=1}^K P(u) V_1(s_0; C_u, \pi_e) = \sum_{u=1}^U P(\sigma(u)) V_1(s_0; C_{\sigma(u)}, \pi_e)$. That is, the result of the sum is independent of the order of its terms $P(u) \hat{V}_1(s_0; C_u, \pi_e)$. So, we assume WLOG that we recover the true cluster labels.

Upon obtaining the confounder labels $u_n$ in each trajectory, we can estimate $P(u)$ with $\hat{P}(u) := \frac{1}{N_{traj}}\sum_n \mathbbm{1}(u_n=u)$ via label proportions. By a simple application of Hoeffding's inequality, there is another function $N_1(\delta, \alpha)$ so that for $n \geq N_1(\delta/U, \alpha)$, the weights satisfy $|\hat{P}(u) - P(u)| \leq \alpha$ for all $u$ with probability at least $1-\delta$. 

We use $|ab-cd| \leq |b||a-c| + |c||b-d|$ to conclude that for $n \geq N_1(\delta/U, \epsilon/2H)$, we have the following bound with probability at least $1-\delta$. 
\begin{equation}\label{eqn:value-deviation-bound}
    |V_1(s_0; \pi_e) - \hat{V}_1(s_0; \pi_e)| \leq \frac{\epsilon}{2H} \max_u \hat{V}_1(s_0; C_u, \pi_e) + \max_u (P(u) |\Delta(u)|) \leq \frac{\epsilon}{2} + \max_u |\Delta(u)|
\end{equation}

where $\Delta(u) := V_1(s_0; C_u, \pi_e) - \hat{V}_1(s_0; C_u, \pi_e)$. 

So, whenever we have exact clustering, there is a function $N_2(\delta, \epsilon, b)$ so that $|\Delta(u)| < \epsilon$ for all $u$ outside of a set of probability $\delta$ whenever $\sum_n \ind(u_n=u) \geq N_2(\delta/U, \epsilon, b)$. By Hoeffding's inequality from above, $\sum_n \ind(u_n=u) \geq n(P(u) - \alpha) \geq nP(u)/2$ for $\alpha \leq \min_u P(u)/2$.

So, for $n \geq \max \left(N\left(\frac{\delta}{3}\right), N_1\left(\frac{\delta}{3U},\min\left(\frac{\epsilon}{2H}, \frac{\min_u P(u)}{2}\right)\right), \frac{2}{\min_u P(u)}N_2\left(\frac{\delta}{3U}, \frac{\epsilon}{2}, b \right) \right)$, we get that $|V_1(s_0; \pi_e) - \hat{V}_1(s_0; \pi_e)| \leq \epsilon$

Note that $N(\delta/3) = U^2SN_0\log(3/\delta)$ and $N_1\left(\frac{\delta}{3U},\min\left(\frac{\epsilon}{2H}, \frac{\min_u P(u)}{2}\right)\right) = \frac{2\log(3U/\delta)}{\min(\epsilon^2, \min_u P(u)^2)}$. This gives us our final bound.

\end{proof}

\newpage

\section{The Necessity of the Horizon Being $O(t_{mix})$}\label{sec:tmix-lower-bound}
We showed in Section~\ref{sec:clusterMDP} that under Assumptions~\ref{assump:global_u}, \ref{assump:mixing}, \ref{assump:model_sep} and \ref{assump:concentrability}, Algorithm \ref{algo:clusterOPE} provides a point estimate of the policy's value with provable sample complexity guarantees. The only additional requirement was that $H \geq H_0 t_{mix} \log n$. We claim that the $t_{mix}$ dependence is not an artifact of the clustering method used. In fact, the theorem below shows that if $H \leq \Tilde{O}(t_{mix})$, clustering and value estimation can be arbitrarily bad even when $t_{mix}$ is small. It essentially produces an example with logarithmically small $t_{mix}$ where the confounders cannot be identified for $H \leq \Tilde{O}(t_{mix})$. We prove it in Appendix~\ref{sec:tmix-lower-bound}. 
We state Theorem~\ref{thm:tmix-lower-bound} below.

\begin{restatable}[Necessity of $H \geq \Omega(t_{mix})$]{thm}{tmixLowerBound}\label{thm:tmix-lower-bound}
There exist globally confounded MDPs $\M_1$ and $\M_2$ and a behavior policy $\pi_b$ with induced mixing time $t_{mix} = O(\log S)$ so that for $H \leq \Tilde{O}(t_{mix})$, trajectories from confounders in both MDPs have the same distribution. Furthermore, there exists a stationary evaluation policy $\pi_e$ and a starting state $s$ so that $|V_1^{\pi_e}(s, \M_1) - V_1^{\pi_e}(s, \M_2)| = \Omega(H)$.
\end{restatable}

\begin{proof}

We construct two MDPs which satisfy all our assumptions, but have the same distribution over a horizon less than $t_{mix}$ and thus cannot be distinguished. We will also note that given the reward structure, under a different starting distribution, the MDPs will have value functions differing by $O(H)$. 

The intuition is that the state space is an $n$-dimensional Boolean hypercube with an extra rewarding state $s_r$, thought of as a "twin" to $(1,1,\dots 1)$. If one identifies $s_r$ to $(1,1,\dots 1)$, then $a=1$ pushes states to have more ones while $a=2$ pushes states to have more zeros, and the actions taken with probability $1/2$ combine to produce a lazy random walk on the Boolean hypercube. Depending on which MDP one is in, $s_r$ and $(1,1, \dots 1)$ have proportional transition dynamics, with different levels of "traffic." Controlling this "traffic" allows us to control the rewards of a different evaluation policy in the MDPs, because we choose all states besides $s_r$ to have $0$ reward.

\paragraph{Environments:} 
\begin{itemize}
    \item Consider $S = \{0,1\}^n\cup\{s_r\}$, $A = \{1,2\}$, $U = \{1,2\}$, horizon $H$.
    \item Rewards: $r(s=s_r, a) = 1$ for any action $a$, otherwise $0$ reward.
    \item Starting state: Let the starting state be $(0,0 \dots 0)$.
    \item Confounders: $\prob(u=1) = \prob(u=2) = \frac{1}{2}$.
\end{itemize}

\textbf{Transitions:} We describe the transition structure below. Pick a parameter $p_{i,j} \in [0,1]$ for MDP $\M_i$ and confounder $u=j$, whose role will be clear below. For both MDPs $\M_1$ and $\M_2$ and both confounders $u=1,2$, consider the following transition structure. 
\begin{itemize}
    \item {Under $a=1$:} Consider $s \neq s_r$, and let it have $k > 1$ zeros. Pick one of the zeros with probability $\frac{1}{n}$ each and change it to a $1$, doing nothing and staying in $s$ with probability $\frac{n-k}{n}$. If $s$ has exactly $1$ zero, then for MDP $\M_i$ and confounder $u=j$, let $s$ transition to $s_r$ with probability $\frac{p_{i,j}}{n}$, to $(1,1, \dots 1)$ with probability $\frac{1-p_{i,j}}{n}$ and stay at $s$ with probability $1-\frac{1}{n}$. Fix $p_{2,1} = p_{2,2} = \frac{1}{2}$. If $s = s_r$, then in $\M_i$ and confounder $u_j$, move to $(1,1, \dots 1)$ with probability $1-p_{i,j}$, staying with probability $p_{i,j}$. If $s = (1,1, \dots 1)$, then in $\M_i$ and confounder $u_j$, move from to $s_r$ with probability $p_{i,j}$, staying with probability $1-p_{i,j}$.
    \item {Under $a=2$:} Consider $s \neq s_r$, and let it have $k > 0$ zeros. Pick one of the ones with probability $\frac{1}{n}$ each and change it to a zero, doing nothing and staying in $s$ with probability $\frac{k}{n}$. If $s = s_r, (1,1, \dots 1)$, then let it transition to a state with a single zero with probability $\frac{1}{n}$.
\end{itemize}

\textbf{Behavior policies:} In both MDPs, choose the same policy $\pi(a \mid s) = \frac{1}{2}$ for all $a,s$. One can check that the occupancies of $s_r$ and $(1,1, \dots 1)$ are only non-zero together and always have the ratio $p_{i,j}/(1-p_{i,j})$ in MDP $\M_i$ and confounder $u=j$. This will thus also hold in the stationary distribution. Note that while in general, identifying states in a Markov chain does not create a Markov chain, this is true if two states always have the same ratio of occupancies. Additionally, since the occupancy ratios are fixed, for any MDP and confounder in our system, the TV distance between the distribution of the system and at any time $t$ from the stationary distribution is the same if we identified $s_r$ and $(1,1, \dots 1)$. Thus, this system has the same mixing time as it would if we identified $s_r$ and $(1,1, \dots 1)$. 

Notice that the transition structure of the induced Markov chains in both MDPs after identifying $s_r$ and $(1,1, \dots 1)$ is identical, and in fact it is the same as picking a bit in a state uniformly at random and flipping it with probability $1/2$, doing nothing otherwise. This is in fact the same as the lazy random walk on the Boolean hypercube in \cite{levin2017markov}. We thus know from \cite{levin2017markov} that both induced Markov chains have the same mixing time $t_{mix} = O(n \log n)$. Let $k$ be a constant so that $t_{mix} \leq kn \log n$.

\textbf{Observational indistinguishability:} Consider $H \leq \frac{t_{mix}}{4k\log(t_{mix})} \leq \frac{n}{4}$. Since the MDPs have identical transition structures for $s \neq s_r$ with $s$ having $2$ or more zeros, and no state can have fewer than $2$ zeros after less than $\frac{n}{4}$ bit flips starting from the starting state $(0,0 \dots 0)$, trajectories generated under either MDP and either confounder have the same probability.

In particular, the confounders are observationally indistinguishable in either MDP and cannot be clustered even with infinite observations, even though transitions differ in $n+2$ of the states with $\Delta > \max(|1-2p_{i,j}|, |p_{i,1} - p_{i,2}|) > 0$. Moreover, the MDPs themselves are observationally indistinguishable as well.

\textbf{Evaluation policy:} One can produce many examples of an evaluation policy $\pi_e$ so that there is a state $s$ with $V^{\pi_e}_{1,i}(s)$ very different across the two MDPs. Here we present a trivial one. Consider $\pi_e(a=1 \mid s, u) = 1$ for all $s,u$. 

\textbf{Policy values:} Let us say that we intend to find $V^{\pi_e}_{1,i}((1,1,\dots 1))$. Notice that in the first step in confounder $u=j$ and MDP $\M_i$, the distribution of states will be $\prob(s_r) = p_{i,j}$ and $\prob((1,1,\dots 1)) = 1-p_{i,j}$ and stays that way for all future steps. This means that $V^{\pi_e}_{1,i}((1,1,\dots 1)) = \left(\sum_{j=1}^2 \frac{p_{i,j}}{2}\right)(H-1)$ in MDP $\M_i$.

The difference in values is given by $|V^{\pi_e}_{1,1}((1,1,\dots 1)) - V^{\pi_e}_{1,2}((1,1,\dots 1))| = (H-1)(p_{1,1} + p_{1,2} - p_{2,1} - p_{2,2})$. We arbitrarily instantiate our parameters to be say $p_{1,1} = 1-\frac{1}{100}$, $p_{1,2} = 1-\frac{2}{100}$, $p_{2,1} = -\frac{1}{100}$, $p_{2,2} = \frac{2}{100}$, to get that 

$$|V^{\pi_e}_{1,1}((1,1,\dots 1)) - V^{\pi_e}_{1,2}((1,1,\dots 1))| = \frac{94}{100}(H-1) = \Omega(H)$$

\end{proof}
\newpage

\section{Policy Optimization under General and Memoryless Confounders}\label{sec:iid_opt}
\subsection{Bounds on Sub-optimality given Optimization Oracles}

Here, we elaborate on the comment at the beginning of Section~\ref{sec:optimization}, where we claim that given error bounds on our value estimate $\hat{V}_1$ and an optimizer for $\hat{V}_1$, we can get suboptimality bounds for the output of the optimizer. Notice the slight change in notation below.

\begin{lemma}
Fix an arbitrary starting distribution $d_0$. If for any policy $\pi$, $|\hat{V}_1(\pi) - {V_1}(\pi)| \leq \epsilon$, then for $\hat{\pi}^* = \argmax_\pi \hat{V}_1(\pi)$ and ${\pi}^* = \argmax_\pi {V}_1(\pi)$, we have that $0 \leq {V}_1({\pi}^*) - {V_1}(\hat{\pi^*}) \leq 2\epsilon$.
\label{lemma:valueBounds}
\end{lemma}

\begin{proof}
Consider the following chain of inequalities.
\begin{align*}
    &{V}_1({\pi}^*) - V_1(\hat{\pi}^*)\\
    = &\ {V}_1({\pi}^*) - \hat{V}_1({\pi}^*) + \hat{V}_1({\pi}^*) - \hat{V}_1(\hat{\pi}^*) + \hat{V}_1(\hat{\pi}^*) - {V}_1(\hat{\pi}^*)\\
    \leq &\ \epsilon + 0 + \epsilon
\end{align*}
Here, the first part of the last inequality holds by our assumption applied to $\pi = {\pi}^*$, while the second part holds by the definition of $\hat{\pi}^*$ as the optimal policy for $\hat{V}_1$. The third part holds by applying our assumption to $\pi = \pi^*$.

Finally, by the definition of $\pi^*$ as the optimal policy for $V_1$, $V_1({\pi}^*) - V_1(\hat{\pi^*}) \geq 0$. Combining these, we have our results.
\end{proof}

\subsection{Gradient Ascent on the Lower Bound}

\begin{algorithm}[h]
	\centering
	\caption{Gradient Ascent on Differentiable Lower Bounds for Policy Improvement under Confounding}
	\begin{algorithmic}[1]
		\STATE \textbf{input: } decaying learning rate $\eta_t$, $\pi_\theta$.
	    \FOR{$t = 1, ..., N$}
		\STATE \textbf{run subroutine: } obtain differentiable lower bound $V_1(s_0 ; \pi_\theta)$ on $\pi_\theta$ via Alg.~\ref{algo:projOPE}, Alg.~\ref{algo:MBRelax}, or Alg.~\ref{algo:cFQE} 
		\STATE \textbf{update: } $\theta \leftarrow \theta + \eta_t \cdot \nabla_\theta V_1(s_0 ; \pi_\theta)$ 
		\ENDFOR 
		\STATE \textbf{return} $\pi_\theta$
	\end{algorithmic}
\label{algo:polGrad}
\end{algorithm}

This enjoys the following elementary local convergence guarantees. 
\begin{lemma} \label{lemma:maxmindynamics}
If $\nabla_\theta V_1(s_0; \pi_\theta, \prob)$ and $\nabla_\prob V_1(s_0; \pi_\theta, \prob)$ are Lipschitz, every local max-min is a gradient ascent/descent stable point.
\end{lemma}

\begin{lemma} \label{lemma:iidmaxdynamics}
If $V_1(s_0; \pi_\theta, \prob)$ is twice differentiable with a Lipschitz continuous gradient, its saddle points are a strict-saddle, and one waits for the inner minimization to converge in each iteration, in the limit of infinite trajectories the procedure converges to a local maxima of $V_1(s_0; \pi_\theta, \prob)$.
\end{lemma}

The first result follows from Section 2 in \citet{constantinos2018minmax}, given the knowledge that $V_1(s_0 ; \pi_\theta, \prob)$, being constructed from translations and matrix multiplications, is smooth, and therefore so are its gradients. The second result follows from \citet{lee2016graddesc}.

\newpage

\section{Policy Optimization under Global Confounders}
\label{sec:clustering-pg}

\begin{algorithm}[h]
	\centering
	\caption{Clustering-Based Policy Gradient}
	\begin{algorithmic}[1]
		\STATE \textbf{input: } Number of clusters $U$, clustering algorithm \texttt{cluster()}, offline policy gradient estimator \texttt{gradient()}, learning rate $\eta$, initial policy parameters $\theta_0$.
		\STATE \textbf{run subroutine:} Perform clustering on trajectories with clustering algorithm \texttt{cluster()}, obtain clusters $C_1,...,C_K$.
		\STATE  Obtain cluster weight estimates $\hat{P}(u) := \frac{|C_u|}{N_{traj}}$.
        \FOR{$t=1,...,T$:}
        \STATE \textbf{run subroutine:} Use offline policy gradient estimator \texttt{gradient()} to estimate  $Z_i(\theta_t) = \nabla_\theta V_1(s_0 ; u_i, \pi_{\theta_t})$ for each cluster $C_i$, obtaining $\hat{Z}_i(\theta_t)$.
        \STATE Obtain gradient estimate of $Z(\theta_t) = \nabla_\theta V_1(s_0 ; \pi_{\theta_t})$ with $\hat{Z}(\theta_t) = \sum_{u=1}^U \hat{P}(u_i) \hat{Z}_i(\theta_t)$.
        \STATE Update $\theta_{t+1} := \theta_t - \eta \hat{Z}(\theta_t)$.
        \ENDFOR
		\STATE \textbf{return: } Output the final policy $\pi_{\theta_{T+1}}$.
	\end{algorithmic}
\label{algo:clusterOpt}
\end{algorithm}

We now recall Theorem~\ref{thm:clustering-pg-endtoend} below. We remind the reader that like Theorem~\ref{thm:cluster-ope-sample-complexity}, the theorem below holds when $H \geq H_0t_{mix}\log n$. 


\ClusteringPG*

To prove this, we first provide a high-probability guarantee for the overall gradient estimate across all clusters analogous to that of Theorem \ref{thm:cluster-ope-sample-complexity} for OPE. This is proved in Section~\ref{subsec:clustering-pg-proof}.

\begin{thm}
\label{thm:clustering-pg}

When Assumptions~\ref{assump:global_u}, \ref{assump:mixing}, \ref{assump:model_sep} and \ref{assump:concentrability} are satisfied, there are constants $H_0$, $N_0$ depending polynomially on $\frac{1}{\alpha}, \Delta, \frac{1}{\min_u P(u)}, \log(1/\delta)$, so that for $n$ trajectories of length $H \geq H_0t_{mix}\log(n)$, if we use the EOPPG offline policy gradient estimator from \cite{kallus2020statistically},

$$n \geq \max \left(U^2SN_0\log(3/\delta), \frac{8\log(6U/\delta)}{\min\{\epsilon^2/L^2, \min_u P(u)^2\}}, \frac{C}{\min_u P(u)}\frac{H^4 \log(nU/\delta)}{\epsilon^2} \right)$$

then $||Z(\theta) - \hat{Z}(\theta)|| \leq \epsilon$ with probability $1-\delta$ for some constant $C$.
    
\end{thm}

The following result for the convergence of unconstrained gradient descent is effectively Theorem 11 in \citet{kallus2020statistically}, combined with the bound in Theorem \ref{thm:clustering-pg}. We repeat the proof in Section~\ref{subsec:clustering-pg-convergence-proof} for completeness.  

\begin{thm}
\label{thm:clustering-pg-convergence}   
Assume $V_1(s_0 ; u, \pi_\theta)$ and $V_1(s_0 ; \pi_\theta)$ are differentiable and $M$-smooth in $\theta$ for all $u \in U$, and the learning rate $\eta < \frac{1}{4M}$. Then, if the number of trajectories $n$ satisfies the condition in Theorem \ref{thm:clustering-pg}, the iterates $\theta_t$ from Algorithm \ref{algo:clusterOpt} offer
$$\frac{1}{T}\sum_{t=1}^T ||\nabla_\theta Z(\theta_t)||^2 = \frac{1}{T}\sum_{t=1}^T ||\nabla_\theta V_1(s_0 ; \pi_{\theta_t})||^2 \leq \frac{4}{\eta T}\left(V_1(s_0 ; \pi_{\theta^*}) - V_1(s_0 ; \pi_{\theta_1})\right) + 3\epsilon^2$$
\end{thm}

The result of Theorem \ref{thm:clustering-pg-endtoend} then follows immediately from the two results above. The only additional observation needed is that since $V_1$ is Lipschitz, it is bounded in a compact domain and so the first term in Theorem~\ref{thm:clustering-pg-convergence} is $O(1/n^\beta) \leq O(1/n)$.

\paragraph{Another Formulation of Policy Optimization} Notice that the nature of the global confounder assumption permits another kind of policy optimization. One can optimize $U$ different policies, one for each value of the confounder, with standard off-policy improvement methods. To deploy them, one will have to identify the confounder online, which is a nontrivial problem in itself. One avenue is to first deploy each of the $U$ behavior policy components in any order for $O(t_{mix})$ time each, and then attempt to identify the confounder using the classification algorithm in \cite{ambuj2022mixmdp}. If the classification algorithm successfully classifies trajectories generated in this way, we can achieve the optimal reward thereafter by deploying the optimal policy for the confounder in question. 

\subsection{Proof of Theorem \ref{thm:clustering-pg}}\label{subsec:clustering-pg-proof}

\begin{proof}
Note that $\nabla_\theta V_1(s_0 ; \pi_\theta) = \mathbb{E}_{u}[\nabla_\theta V_1(s_0 ; u, \pi_\theta)] = \sum_u P(u) \nabla_\theta V_1(s_0 ; u, \pi_\theta)$.

Using the clustering guarantee from \citet{ambuj2022mixmdp} rephrased in Section~\ref{sec:clustering-ope}, we know that there are numbers $N_0$ and $H_0$ so that given $n \geq U^2SN_0\log(1/\delta)$ trajectories of length $H \geq H_0t_{mix}\log(n)$, we recover clusters $C_1,...,C_U$ consisting of trajectories with the same confounders with probability at least $1-\delta$. Recall that $N_0$ and $H_0$ are not explicitly dependent on $S, A$ and $t_{mix}$, but could depend on the model.

Write $Z(\theta) = \nabla_\theta V_1(s_0 ; \pi_\theta)$, $Z_i(\theta) = \nabla_\theta V_1(s_0 ; u_i, \pi_\theta)$ and $\hat{Z}_i(\theta)$ for the estimate of $Z_i(\theta)$ and $\hat{Z}(\theta) = \sum_{i=1}^U \hat{P}(u_i) \hat{Z}_i(\theta)$ for the estimate of $Z(\theta)$. We only identify the confounder labels in each trajectory up to permutation upon obtaining exact clustering, but as above we assume WLOG that we recover the true cluster labels.

Estimate $P(u)$ with $\hat{P}(u) := \frac{1}{N_{traj}}\sum_n \mathbbm{1}(u_n=u)$ via label proportions. By a simple application of Hoeffding's inequality and the union bound, for $n \geq \frac{2\log(2U/\delta)}{\alpha^2}$, the weights satisfy $|\hat{P}(u) - P(u)| \leq \alpha$ with probability at least $1-\delta$.

We can then bound
\begin{align}
    ||Z(\theta) - \hat{Z}(\theta)||
    &= \left\Vert\sum_{i=1}^U \left(P(u_i) Z_i(\theta) - \hat{P}(u_i)\hat{Z_i}(\theta)\right)\right\rVert \\
    &= \sum_{i=1}^U \left\lVert P(u_i) Z_i(\theta) - \hat{P}(u_i)\hat{Z_i}(\theta)\right\rVert \\
    &\leq \sum_{i=1}^U  ||Z_i(\theta)|| (P(u_i) - \hat{P}(u_i)) + \hat{P}(u_i) ||Z_i(\theta) - \hat{Z_i}(\theta)|| \\
    &= \sum_{i=1}^U  ||Z_i(\theta)|| (P(u_i) - \hat{P}(u_i)) + \sum_{i=1}^U \hat{P}(u_i) ||Z_i(\theta) - \hat{Z_i}(\theta)|| \\
    &\leq \alpha \sum_{i=1}^U  ||Z_i(\theta)|| + \sum_{i=1}^U \hat{P}(u_i) ||Z_i(\theta) - \hat{Z_i}(\theta)|| \\
    &\leq \alpha \sum_{i=1}^U  ||Z_i(\theta)|| + \sum_{i=1}^U 2P(u_i) ||Z_i(\theta) - \hat{Z_i}(\theta)||
\end{align}
where the second inequality holds with high probability and the last inequality holds for sufficiently small $\alpha$. If all $|| Z_i(\theta) - \hat{Z}_i(\theta) || \leq \epsilon/4$ for some $\epsilon > 0$, then we would have $||Z(\theta) - \hat{Z}(\theta)|| \leq  \sum_{i=1}^U 2P(u_i) || Z_i(\theta) - \hat{Z}_i(\theta) || \leq \epsilon/2$.

It remains to bound the error of each $\hat{Z}_i$. Notice that the result of Theorem 7 in \cite{kallus2020statistically} is independent of the gradient update rule or the value of $\theta$ and only depends on the number of samples used to estimate $\hat{Z}^{EOPPG}$. So, it also holds for $\hat{Z}_i$ with $n_i$ samples. Additionally, note that the proof of Theorem 12 in \cite{kallus2020statistically} only uses the supremum of the error over all possible values of $\theta$ and does not use any facts about the gradient update, it follows verbatim for $\hat{Z}_i$ with $n_i$ samples. In particular, with probability at least $1-\delta/U$,
$$\|Z_i(\theta) - \hat{Z}_i(\theta)\|^2 \leq O\left( \frac{H^4 \log(TU/\delta)}{n_i}\right)$$  

and so for $T = n^\beta$, we need $n_i \geq \Omega\left(\frac{H^4\log(nU/\delta)}{\epsilon^2}\right)$ trajectories for $|| Z_i(\theta) - \hat{Z}_i(\theta) || \leq \epsilon$ to hold for all $u_i$ with probability $1-\delta$. To convert this into a bound for $n$, we use Hoeffding's inequality from above in a similar way to the previous proof to find $n_i = \sum_n \ind(u_n=u) \geq n(P(u) - \alpha) \geq nP(u)/2$ for $\alpha \leq \min_u P(u)/2$. We therefore need $n \geq \Omega\left(\frac{1}{\min_u P(u)}\frac{H^4\log(nU/\delta)}{\epsilon^2}\right) $ for the error of each $Z_i$ to be bounded by $\epsilon$ with probability $1-\delta$. 

We then bound $\alpha \sum_{i=1}^U ||Z_i(\theta)|| \leq \epsilon/2$. Let $L$ be a uniform bound over $\theta \in \Theta$ on the magnitude of the gradients $Z(\theta)$ (in the continuous case, this corresponds to a Lipschitz-type assumption on the value functions). It then suffices to require $\alpha \leq \frac{\epsilon}{2L}$.

Splitting the failure probability into $\delta/3$, requiring $\alpha \leq \min_u P(u)/2, \epsilon/2L$, and bounding the error of each $Z_i$ by $\epsilon/4$, we get $||Z(\theta) - \hat{Z}(\theta)|| \leq \epsilon$ with probability $1-\delta$ when

\begin{equation}
    n \geq \Omega\left(\max \left(U^2SN_0\log(1/\delta), \frac{\log(U/\delta)}{\min\{\epsilon^2/L^2, \min_u P(u)^2\}}, \frac{1}{\min_u P(u)}\frac{H^4\log(nU/\delta)}{\epsilon^2}\right) \right)
\end{equation}

\end{proof}

\subsection{Proof of Theorem \ref{thm:clustering-pg-convergence}}\label{subsec:clustering-pg-convergence-proof}

\begin{proof}
    The result is largely analogous to Theorem 11 from \citet{kallus2020statistically}, and in fact, we can transform our problem into theirs and follow their proof. 

    Assume $V_1(s_0 ; u, \pi_\theta)$ and $V_1(s_0 ; \pi_\theta)$ are differentiable and $M$-smooth in $\theta$ for all $u \in U$. Let $f(\theta) = -V_1(s_0;\pi_\theta)$, and $f_i(\theta) = -V_1(s_0 ; u_i, \pi_\theta)$ for each $u_i$. For simplicity, fix the learning rate for all time steps to be some $\eta < \frac{1}{4M}$. By $M$-smoothness, 

    $$f(\theta_{t+1}) \leq f(\theta_t) + \langle \nabla f(\theta_t, \theta_{t+1} - \theta_t \rangle + \frac{M}{2}||\theta_{t+1} - \theta_t||^2.$$

    Define $B_{it} = \hat{Z}_i(\theta) - Z_i(\theta)$ for confounder $u_i$, $B_t = \hat{Z}(\theta) - Z(\theta)$, $w_i = \hat{P}(u_i)$. Observe that $$\theta_{t+1} = \theta_t - \eta \nabla f(\theta_t) - \eta B_t = \theta_t - \eta \sum_i \nabla w_i f(\theta_t) - \eta \sum_i w_i B_{it}. $$
    
    Then, similarly to the proof in \citet{kallus2020statistically}, 
    \begin{align}
        f(\theta_t) - f(\theta_{t+1})
        &\geq -\langle \nabla f(\theta_t), \theta_{t+1} - \theta_t \rangle - \frac{M}{2}||\theta_{t+1} - \theta_t||^2 \\
        &= \eta \langle \nabla f(\theta_t), \nabla f(\theta_t) - B_t \rangle - \frac{\eta^2 M}{2} ||\nabla f(\theta_t) - B_t||^2 \\
        &= \eta ||\nabla f(\theta_t)||^2 + \eta \langle \nabla f(\theta_t), B_t \rangle - \frac{\eta^2 M}{2} ||\nabla f(\theta_t) - B_t||^2 \\ 
        &\geq  \eta ||\nabla f(\theta_t)||^2 - \eta |\langle \nabla f(\theta_t), B_t \rangle| - \frac{\eta^2 M}{2} ||\nabla f(\theta_t) - B_t||^2 \\ 
        &\geq \eta ||\nabla f(\theta_t)||^2 - 0.5\eta (||\nabla f(\theta_t)||^2 + ||B_t||^2) - \eta^2 M ||\nabla f(\theta_t) - B_t||^2 \\
        &\geq 0.25 \eta ||\nabla f(\theta_t)||^2 - 0.5 \eta ||B_t||^2 - 0.25 \eta ||B_t||^2
    \end{align}
    where the second-last inequality uses the parallelogram law and the last inequality uses the fact that $\eta < \frac{1}{4M}$. We then obtain 

    $$ f(\theta_t) - f(\theta_{t+1}) + 0.75 \eta ||B_t||^2 \geq 0.25 \eta ||\nabla f(\theta_t)||^2. $$

    Similarly, by a telescoping sum, 
    $$ (f(\theta_1) - f(\theta^*))/T + \frac{0.75 \eta}{T}\sum_t  ||B_t||^2 \geq \frac{0.25\eta}{T} \sum_t  ||\nabla f(\theta_t)||^2, $$
    $$ (V_1(s_0;\pi_{\theta^*}) - V_1(s_0;\pi_{\theta_1}))/T + \frac{0.75 \eta}{T}\sum_t  ||B_t||^2 \geq \frac{0.25\eta}{T} \sum_t  ||\nabla f(\theta_t)||^2, $$
    $$ \frac{\eta}{T} \sum_t  ||\nabla f(\theta_t)||^2 \leq \frac{4}{T}(V_1(s_0;\pi_{\theta^*}) - V_1(s_0;\pi_{\theta_1})) + \frac{3\eta}{T}\sum_t  ||B_t||^2, $$
    $$ \frac{1}{T} \sum_t  ||\nabla f(\theta_t)||^2 \leq \frac{4}{\eta T}(V_1(s_0;\pi_{\theta^*}) - V_1(s_0;\pi_{\theta_1})) + \frac{3}{T}\sum_t  ||B_t||^2, $$
    $$ \frac{1}{T} \sum_t  ||\nabla f(\theta_t)||^2 \leq \frac{4}{\eta T}(V_1(s_0;\pi_{\theta^*}) - V_1(s_0;\pi_{\theta_1})) + 3 \max_t ||B_t||^2, $$

    and finally by applying Theorem \ref{thm:clustering-pg} for an $n$ that fulfills its conditions for some error threshold $\epsilon$, we obtain

    $$ \frac{1}{T} \sum_t  || Z(\theta_t)||^2 = \frac{1}{T} \sum_t  ||\nabla f(\theta_t)||^2 \leq \frac{4}{\eta T}(V_1(s_0;\pi_{\theta^*}) - V_1(s_0;\pi_{\theta_1})) + 3 \epsilon^2. $$


\end{proof}

\end{document}